\documentclass[letterpaper, 10 pt, conference]{ieeeconf}

\IEEEoverridecommandlockouts                              

\usepackage{graphicx} 
\usepackage{cite}
\usepackage[utf8]{inputenc}
\usepackage{amsmath}
\usepackage[ruled, vlined, linesnumbered]{algorithm2e}
\usepackage{amssymb}
\usepackage{amsfonts}
\usepackage[bb=boondox]{mathalfa}
\usepackage[utf8]{inputenc}
\usepackage[english]{babel}
\usepackage{bm}
\usepackage{accents}
\newtheorem{theorem}{\bf Theorem}

\newtheorem{lemma}{\bf Lemma}

\usepackage[caption=false,font=footnotesize]{subfig}

\usepackage[dvipsnames]{xcolor}
\usepackage{soul}
\usepackage{physics}

\title{\LARGE \bf
Rationally Inattentive Path-Planning via RRT*
}

\author{Jeb Stefan$^{1}$, Ali Reza Pedram$^{2}$, Riku Funada$^{3}$ and Takashi Tanaka$^{3}$ 
\thanks{*This work is supported by Lockheed Martin Corporation.}
\thanks{$^{1}$Odyssey Space Research.
        {\tt\small jeb.stefan@odysseysr.com}.
        $^{2}$Walker Department of Mechanical Engineering, University of Texas at Austin.
        {\tt\small apedram@utexas.edu}. $^{3}$Department of Aerospace Engineering and Engineering Mechanics, University of Texas at Austin.
        {\tt\small riku.funada@austin.utexas.edu} and  {\tt\small ttanaka@utexas.edu}. }%
}

\begin{document}

\maketitle
\thispagestyle{empty}
\pagestyle{empty}

\begin{abstract}

We consider a path-planning scenario for a mobile robot traveling in a configuration space with obstacles under the presence of stochastic disturbances.
A novel path length metric is proposed on the uncertain configuration space and then integrated with the existing RRT* algorithm. The metric is a weighted sum of two terms which capture both the Euclidean distance traveled by the robot and the perception cost, i.e., the amount of information the robot must perceive about the environment to follow the path safely. The continuity of the path length function with respect to the topology of the total variation metric is shown and the optimality of the Rationally Inattentive RRT* algorithm is discussed. Three numerical studies are presented which display the utility of the new algorithm.

\end{abstract}

\section{Introduction}
\label{sec:Intro}
As robots are designed to be more self-reliant in navigating complex and stochastic environments, it is sensible for the strategic execution of perception/cognition tasks to be included in the theory which governs their path-planning \cite{Pendleton2017,Pfeiffer2017,Carlone2019}. Even though the body of work surrounding motion planning techniques has greatly expanded recently, a technological gap remains in the integration of perception concerns into planning tasks \cite{Alterovitz2016}. Mitigating this gap is paramount to missions which require robots to autonomously complete tasks when sensing actions carry high costs (battery power, computing constraints, etc). 

Path-planning is typically followed by feedback control design, which is executed during the path following phase. In the current practice, path-planning and path-following are usually discussed separately (notable exceptions include \cite{kuwata2008motion, van2011lqg, agha2014firm}), and the cost of feedback control (perception cost in particular) is not incorporated in the path-planning phase. 
The first objective of this work is to fill this gap by introducing a novel path cost function which incorporates the expected perception cost accrued during path-following into the planning phase. 
This cost jointly penalizes the amount of sensing needed to follow a path and the distance traveled. Our approach is closely related to the concept of \emph{rationally inattentive (RI) control} \cite{Sims2003} (topic from macroeconomics which has recently been applied in control theory \cite{Shafieepoorfard2016,Shafieepoorfard2013}). The aim of rationally inattentive control is to jointly design the control and sensing policies such that the least amount of information (measured in \emph{bits}) is collected about the environment in order to achieve the desired control. 

The second objective of this work is to integrate the proposed path length function with an existing sampling-based algorithm, such as Rapidly-Exploring Random Trees (RRT) \cite{Lavalle2006}. The RRT algorithm is suited for this problem as it has been shown to find feasible paths in motion planning problems quickly. A modified version of this algorithm, RRT* \cite{Karaman2010}, will be utilized as it has the additional property of being asymptotically optimal. 
We develop an RRT*-like algorithm incorporating the proposed path length function (called the RI-RRT* algorithm) and demonstrate its effectiveness.

While the practical utility of the proposed framework must be thoroughly studied in the future, its expected impact is displayed in Fig.~\ref{fig:intro_example}. 
This figure shows the example of a robot moving through the two-dimensional, obstacle-filled environment.
Path A (red) represents the path from the origin to target location which minimizes the Euclidean distance. However, this path requires a large number of sensor actuations to keep the robot's spatial uncertainty small and avoid colliding with obstacles. Alternatively, Path B (blue) allows for the covariance to safely grow more along the path. Although the Path B travels a greater Euclidean distance to reach the target, it is cheaper in the information-theoretic sense as it requires fewer sensing actions. Therefore, if the perception cost is weighed more than the travel cost, Path B is characterized as the shortest path in the proposed path planning framework. We will demonstrate this effect in a numerical simulation in Section~\ref{sec:2D_mult}.

The proposed concept of rationally inattentive path-planning provides insight into the mathematical modeling of human experts' skills in path planning \cite{marquez2008design}, especially in terms of an efficiency-simplicity trade-off. 
Several path-planning algorithms have been proposed in the literature that are capable of enhancing path simplicity; this list includes potential field approaches \cite{Hwang1992}, multi-resolution perception and path-planning \cite{Kambhampati1986,Hauer2015}, and safe path-planning \cite{Lambert2003,Pepy2006}.
The information-theoretic distance function we introduce in this paper can be thought of as an alternative measure of path simplicity, which may provide a suitable modeling of the human intuition for simplicity in planning. In our standard, a path which requires less sensor information during the path-following phase is more ``simple;" Path B in Fig.~\ref{fig:intro_example} is simpler than Path A, and the simplest path is that which is traceable by an open-loop control policy. 

The contributions of this paper are summarized as follows:
\begin{itemize}
    \item{A novel path cost (RI cost) is formulated which jointly accounts for travel distance and perception cost.}
    \item{The continuity of the path cost with respect to the topology of the total variation metric is shown in the single dimensional case, which is a step forward to guaranteeing the asymptotic optimality of sampling-based algorithms.}
    \item{An RRT*-like algorithm is produced implementing the RI path-planning concept.}
\end{itemize}
\begin{figure}[t]
\centering
\fbox{\includegraphics[width = 0.7\columnwidth]{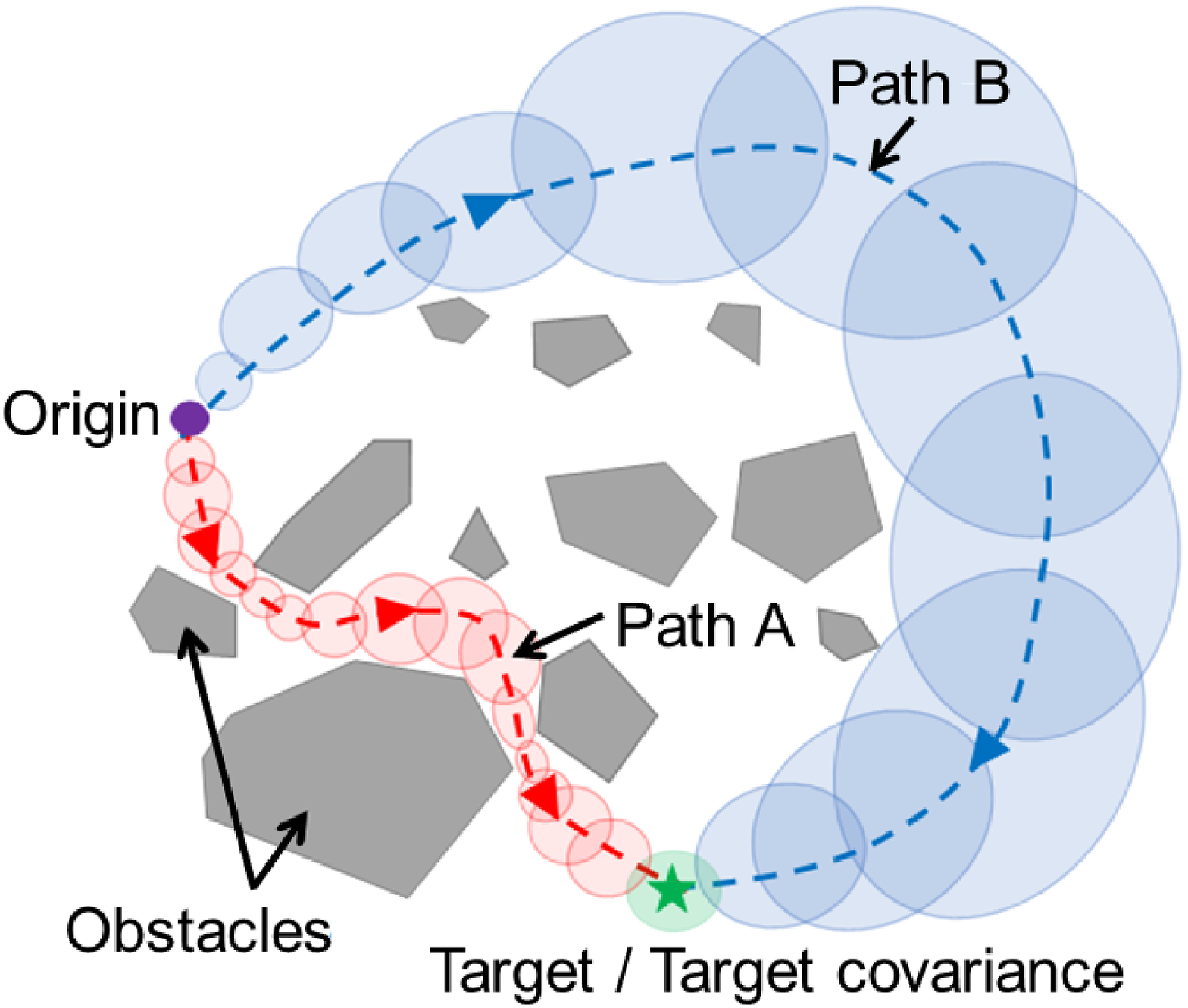}}
\caption{Example of an autonomous robot navigating a two-dimensional configuration space with obstacles. The goal of the robot is to reach the target location. As it moves, the uncertainty of the robot's exact location in the environment grows, represented by the varying sized covariance ellipses.}
\label{fig:intro_example}
\end{figure}

\emph{Notation}: For the purpose of this work, the following definitions for vectors (lower case) and matrices (upper case) hold: $\mathbb{S}^d=\left\{P\in\mathbb{R}^{d \times d}: P \text{ is symmetric.} \right\}, \mathbb{S}_{++}^d=\left\{P\in\mathbb{S}^d: P\succ 0 \right\}$, $\mathbb{S}_\epsilon^d=\left\{P\in\mathbb{S}^d: P\succeq \epsilon I \right\}$, and bold symbols such as $\bm{x}$ represent random variables. 
The vector 2-norm is $\| \cdot \|$ and $\| \cdot \|_{F}$ is Frobenius norm. 
The maximum singular values of a matrix $M$ is denoted by $\bar{\sigma}(M)$. 

\section{Preliminary Material}
\label{sec:Prelim}
In this paper, we consider a path-planning problem for a mobile robot with dynamics given by model (\ref{eq:ito}). Let $\bm{x}(t)$ be a $\mathbb{R}^d$-valued random process representing the robot's position at time $t$, given by the controlled Ito process:
\begin{equation}
\label{eq:ito}
\begin{split}
    &d \bm{x}(t) = \bm{v}(t)dt+W^{\frac{1}{2}} d \bm{b}(t),
\end{split}
\end{equation}
with $\bm{x}(0)\sim\mathcal{N}(x_0, P_0)$ and $t\in [0,T]$. Here, $\bm{v}(t)$ is the velocity input command, $\bm{b}(t)$ is the $d$-dimensional standard Brownian motion, and $W$ is a given positive definite matrix used in modeling the process noise intensity. We assume that the robot is commanded to travel at a unit velocity (i.e., $\|\bm{v}(t)\|=1$). Let $\mathcal{P}=(0=t_0<t_1< \cdots < t_N=T)$ be a partition of $[0,T]$, which must not necessarily be of equal spacing. Time discretization of \eqref{eq:ito} based on the Euler-Maruyama method \cite{Kloeden1992} yields:
\begin{equation}
\label{eq:euler}
    \bm{x}(t_{k+1})=\bm{x}(t_k)+\bm{v}(t_k)\Delta t_k+\bm{n}(t_k),
\end{equation}
where $\Delta t_k=t_{k+1}-t_k$ and $\bm{n}(t_k)\sim\mathcal{N}(0, \Delta t_k W)$.
Introducing a new control input $\bm{u}(t_k):=\bm{v}(t_k)\Delta t_k$ and applying the constraint $\|\bm{v}(t_k)\|=1$, \eqref{eq:euler} can be written as:
\begin{equation}
\label{eq:euler2}
\begin{split}
    &\bm{x}(t_{k+1})=\bm{x}(t_k)+\bm{u}(t_k)+\bm{n}(t_k),
\end{split}
\end{equation}
with $\bm{n}(t_k)\sim\mathcal{N}(0, \|\bm{u}(t_k)\|W)$. Due to the unit velocity assumption above, the time intervals $\Delta t_k, k=0, 1, 2, \cdots$ are determined once the command sequence $\bm{u}({t_0}), \bm{u}(t_1), \bm{u}(t_2), \cdots$ is formalized. Since the physical times $t_k$ do not play significant roles in our theoretical development in the sequel, it is convenient to rewrite \eqref{eq:euler2} as the main dynamics model of this work:
\begin{equation}
\label{eq:euler3}
\begin{split}
    & \bm{x}_{k+1}=\bm{x}_k+\bm{u}_k+\bm{n}_k, \; \bm{n}_k\sim\mathcal{N}(0, \|\bm{u}_k\|W).
\end{split}
\end{equation}

Let the probability distributions of the robot position at a given time step $k$ be parametrized by a Gaussian model $\bm{x}_k\sim\mathcal{N}(x_k, P_k)$, where $x_k\in\mathbb{R}^d$ is the nominal position and $P_k\in \mathbb{S}_{++}^d$ is the associated covariance matrix (with $d$ being the dimension of the configuration space). In this paper, we consider a path-planning framework in which the sequence  $\{(x_k, P_k) \}_{k\in\mathbb{N}}$ is scheduled. Following \cite{lambert2003safe, pepy2006safe}, the product space $\mathbb{R}^d \times \mathbb{S}_{++}^d$ is called the \emph{uncertain configuration space}. In what follows, the problem of finding the shortest path in the uncertain configuration space with respect to a novel information-theoretic path length function is formulated.

First, an appropriate directed distance function from a point $(x_k, P_k)\in \mathbb{R}^d \times \mathbb{S}_{++}^d$  to another $(x_{k+1},P_{k+1})\in \mathbb{R}^d \times \mathbb{S}_{++}^d$ is introduced. This function is interpreted as the cost of steering the random state variable $\bm{x}_k\sim\mathcal{N}(x_k, P_k)$ to $\bm{x}_{k+1}\sim\mathcal{N}(x_{k+1}, P_{k+1})$ in the next time step under the dynamics provided by \eqref{eq:euler3}. In order to implement the rational inattention concept, we formulate this cost as a weighted sum of the control cost $\mathcal{D}_{\text{cont}}(k)$ and the information cost $\mathcal{D}_{\text{info}}(k)$ in achieving each state transition.

\subsection{Control Cost}
\label{sec:Min_ener}
The control cost is simply the commanded travel distance in the Euclidean metric:
\begin{equation}
\mathcal{D}_{\text{cont}}(k):=\|x_{k+1}-x_k\|.
\end{equation}

\subsection{Information Cost}
\label{sec:Info_theory}

Jointly accounting for both the control efficiency and sensing simplicity in planning necessitates the formulation of a metric that captures the information acquisition cost required for path following. We utilize the information gain (entropy reduction) for this purpose. 

Assume that the control input $\bm{u}_k=x_{k+1}-x_k$ is applied to \eqref{eq:euler3}. 
The propagation of the prior covariance during the movement of the robot, over the time interval $[t_k, t_{k+1})$, is denoted as $\hat{P}_{k}=P_k+\|x_{k+1}-x_k\|W$. At time $t_{k+1}$, the covariance is ``reduced" to $P_{k+1}(\preceq \hat{P}_{k})$ by utilizing a sensor input. The minimum information gain (minimum number of \emph{bits} that must be contained in the sensor data) for this transition is:
\begin{equation}
    \mathcal{D}_{\text{info}}(k) = \frac{1}{2}\log_2\det \hat{P}_{k} - \frac{1}{2}\log_2\det P_{k+1}. \label{eq:info_gain1}
\end{equation} 
The notion of an ``optimal" sensing signal which reduces $\hat{P}_{k}$ to $P_{k+1}$ has been previously discussed in \cite{tanaka2016semidefinite} in the context of optimal sensing in filtering theory. The information cost function $\mathcal{D}_{\text{info}}(k)$ in \eqref{eq:info_gain1} is well-defined for the pairs $(P_k, P_{k+1})$ satisfying $P_{k+1}\preceq \hat{P}_{k}$. For those pairs which do not satisfy $P_{k+1}\preceq \hat{P}_{k}$, we generalize \eqref{eq:info_gain1} as:
\begin{equation}
\begin{split}
    \mathcal{D}_{\text{info}}(k)=&\min_{Q_{k+1}\succeq 0} \quad \frac{1}{2}\log_2\det \hat{P}_{k}-\frac{1}{2}\log_2\det Q_{k+1} \label{eq:d_info_general} \\
    & \quad \text{s.t. } \quad Q_{k+1} \preceq P_{k+1}, \;\; Q_{k+1} \preceq \hat{P}_{k}.
\end{split}
\end{equation}
Notice that \eqref{eq:d_info_general} takes a non-negative value for any given transition from an origin $(x_k, P_k)$ to destination $(x_{k+1}, P_{k+1})$.  However, \eqref{eq:d_info_general} is an implicit function involving a convex optimization problem in its expression (more precisely, the max-det problem \cite{vandenberghe1998determinant}). To see why \eqref{eq:d_info_general} is an appropriate generalization of \eqref{eq:info_gain1}, consider a two-step procedure $\hat{P}_{k}\rightarrow Q_{k+1}\rightarrow P_{k+1}$ to update the prior covariance $\hat{P}_{k}$ to the posterior covariance $P_{k+1}$. In the first step, the uncertainty is ``reduced" from $\hat{P}_{k}$ to satisfy both $Q_{k+1}\preceq \hat{P}_{k}$ and $Q_{k+1}\preceq P_{k+1}$. The associated information gain (the amount of telemetry data) is $\frac{1}{2}\log_2\det \hat{P}_{k}-\frac{1}{2}\log_2\det Q_{k+1}$. In the second step, the covariance $Q_{k+1}$ is ``increased" to $P_{k+1}(\succeq Q_{k+1})$. This step incurs no information cost, since the location uncertainty can be increased simply by ``deteriorating" the prior knowledge. The max-det problem \eqref{eq:d_info_general} can then be interpreted as finding the optimal intermediate step $Q_{k+1}$ which minimizes the information gain in the first step.

\subsection{Total Cost}
\label{sec:Total_cost}
The cost to steer a random state variable $\bm{x}_k\sim\mathcal{N}(x_k, P_k)$ to $\bm{x}_{k+1}\sim\mathcal{N}(x_{k+1}, P_{k+1})$ is a weighted sum of $\mathcal{D}_{\text{cont}}(k)$ and $\mathcal{D}_{\text{info}}(k)$.
Introducing $\alpha>0$, the total RI cost is:
\begin{equation}
\label{eq:def_D}
\begin{split}
    &\mathcal{D}(x_k, x_{k+1}, P_k, P_{k+1}) := \mathcal{D}_{\text{cont}}(k)+\alpha \mathcal{D}_{\text{info}}(k) \\
    & \quad =\min_{Q_{k+1}\succ 0} \quad \|x_{k+1}-x_k\| \\
    & \qquad \qquad \qquad +\frac{\alpha}{2}\left[\log_2\det\hat{P}_k -\log_2\det Q_{k+1} \right] \\
    & \qquad \qquad \text{ s.t. } \;\; Q_{k+1} \preceq P_{k+1}, \; \; Q_{k+1} \preceq \hat{P}_k.
\end{split}
\end{equation}
By increasing $\alpha$, more weight is placed on the amount of information which must be gained compared to the distance traversed. Note that the information cost $\mathcal{D}_{\text{info}}$ is an asymmetric function, so that transitioning $(x_{1},P_{1}) \rightarrow (x_{2},P_{2})$ does not return the same cost as $(x_{2},P_{2}) \rightarrow (x_{1},P_{1})$. 

\section{Problem Formulation}
\label{sec:Formulation}
Having introduced the RI cost function \eqref{eq:def_D}, it is now appropriate to introduce the notion of path length. Let $\gamma: [0,T]\rightarrow \mathbb{R}^d\times \mathbb{S}_{++}^d$, $\gamma(t)=(x(t), P(t))$ be a path. The RI length of a path $\gamma$ is defined as:
\begin{equation}
    c(\gamma):=\sup_\mathcal{P} \sum_{k=0}^{N-1} \mathcal{D}\left(x(t_k), x(t_{k+1}), P(t_k), P(t_{k+1})\right), \nonumber
\end{equation}
where the supremum is over the space of partitions $\mathcal{P}$ of $[0,T]$. If $\gamma(t)$ is differentiable and $W\succeq \frac{d}{dt}P(t)\; \forall \; t\in [0,T]$, then it can be shown that:
\begin{equation}
    c(\gamma) = \int_{0}^{T}\left[ \left\|\frac{d}{dt}x(t) \right\| + \frac{\alpha}{2}\text{Tr}\left( W - \frac{d}{dt}P(t) \right)P^{-1}(t) \right]dt \nonumber
\end{equation}

\subsection{Topology on the path space}
In this subsection, we introduce a topology for the space of paths $\gamma: [0,T]\rightarrow \mathbb{R}^d\times \mathbb{S}_{++}^d$, which is necessary to discuss continuity of $c(\gamma)$. The space of all paths $\gamma: [0,T]\rightarrow \mathbb{R}^d\times \mathbb{S}_{++}^d$ can be thought of as a subset (convex cone) of the space of \emph{generalized paths} $\gamma: [0,T]\rightarrow \mathbb{R}^d\times \mathbb{S}^d$. The space of generalized paths is a vector space on which addition and scalar multiplication exist and are defined as  $(\gamma_1+\gamma_2)(t)=(x_1(t)+x_2(t), P_1(t)+P_2(t))$ and $\alpha \gamma(t)=(\alpha x(t), \alpha P(t))$ for $\alpha\in\mathbb{R}$, respectively. Assuming that a path can be partitioned such that $\mathcal{P}=(0=t_0<t_1<\cdots < t_N=T)$, the variation $V(\gamma; \mathcal{P})$ of a generalized path $\gamma$ with respect to the choice of $\mathcal{P}$ is given by:
\begin{equation} \label{eq:variation}
    V(\gamma; \mathcal{P}):=\|x(0)\|+\bar{\sigma}(P(0))+\sum_{k=0}^{N-1}\Bigl[\|\Delta x_k\|+\bar{\sigma}(\Delta P_k) \Bigr] \nonumber
\end{equation}
where $\Delta x_k = x(t_{k+1})-x(t_k)$, and $\Delta P_k=P(t_{k+1})-P(t_k)$. Utilizing the above definition for the variation of a path, the total variation of a generalized path $\gamma$ corresponds to the partition $\mathcal{P}$ which results in the supremum of the variation:
\begin{equation}
    |\gamma|_{\text{TV}}:=\sup_{\mathcal{P}}V(\gamma; \mathcal{P}). \nonumber
\end{equation}
Notice that $|\cdot|_{\text{TV}}$ defines a norm on the space of generalized paths. 
The following relationship holds between $|\gamma|_{\text{TV}}$ and
\begin{equation}
\|\gamma\|_\infty:=\sup_{t\in[0,T]} \|x(t)\|+\bar{\sigma}(P(t)). \nonumber
\end{equation} 
\begin{lemma} \cite[Lemma 13.2]{carothers2000real} 
\label{lem:ineq}
For a given path $\gamma$ with partitioning $\mathcal{P}$ the following inequality holds:
\begin{equation}
    \|\gamma\|_\infty \leq |\gamma|_{\text{TV}}. \nonumber
\end{equation}
\end{lemma}
\begin{proof}
    See Appendix~\ref{sec:AppenI} for proof.
\end{proof}

In what follows, we assume on the space of generalized paths $\gamma: [0,T]\rightarrow \mathbb{R}^d\times \mathbb{S}^d$ the topology of total variation metric $|\gamma_1-\gamma_2|_{\text{TV}}$, which is then inherited to the space of paths $\gamma: [0,T]\rightarrow \mathbb{R}^d\times \mathbb{S}_{++}^d$. We denote by $\mathcal{BV}[0, T]$ the space of paths $\gamma: [0,T]\rightarrow \mathbb{R}^d\times \mathbb{S}_{++}^d$ such that $|\gamma|_{\text{TV}}<\infty$.
In the next subsection, we discuss the continuity of the RI path cost $c(\cdot)$ in the space $\mathcal{BV}[0, T]$.

\subsection{Continuity of RI Cost Function} \label{sec:continuity}
The continuity of RI path cost function plays a critical role in determining the theoretical guarantees we can provide when we use sampling-based algorithms to find the shortest RI path. Specifically, the asymptotic optimality (the convergence to the path with the minimum cost as the number of nodes is increased) of RRT* algorithms \cite{Karaman2010}, the main numerical method we use in this paper, expects the continuity of the path cost function. Showing that the RI cost function (\ref{eq:def_D}) is continuous requires additional derivation which is shown via Theorem \ref{theo:continuity}.
\begin{theorem}
\label{theo:continuity}
When $d=1$, the path cost function $c(\cdot)$ is continuous in the sense that for every $\gamma\in\mathcal{BV}[0,T]$, $\gamma : [0,T]\rightarrow \mathbb{R}^1 \times \mathbb{S}_{2\epsilon}^1$, and for every $\epsilon_0>0$, there exists $\delta >0$ such that
\begin{equation}
    |\gamma'-\gamma|_{\text{TV}}<\delta \quad \Rightarrow \quad  |c(\gamma')-c(\gamma)|<\epsilon_0. \nonumber
\end{equation}
\end{theorem}
\begin{proof}
    See Appendix~\ref{sec:AppenII} for proof.
\end{proof}

Before discussing the required modifications for implementing RRT* algorithm with RI cost in Section~\ref{sec:RRT_main}, we first characterize the shortest RI path in the obstacle-free space, and then formally define the shortest RI path problem in obstacle-filled spaces in the following subsections.

\subsection{Shortest Path in Obstacle-Free Space} \label{sec:Prelim-D}
In obstacle-free space, it can be shown that the optimal path cost between $z_1=(x_1,P_1)$ and $z_2=(x_2,P_2)$ is equal to $\mathcal{D}(z_1, z_2)$. In other words, the triangular inequality $\mathcal{D}(z_1, z_2) \leq \mathcal{D}(z_1, z_{int})+\mathcal{D}(z_{int}, z_2)$ holds. This means it is optimal for the robot to follow the direct path from $x_1$ to $x_2$ without sensing, and then make a measurement at $x_2$ to shrink the uncertainty from $\hat{P}_1$ to $P_2$. In what follows, we call such a motion plan the ``move-and-sense" strategy. 
The optimality of the move-and-sense path for one-dimensional geometric space is shown in Appendix~\ref{sec:AppenIII}. We confirm this optimality by simulation in Section~\ref{sec:1D}, where the move-and-sense path is the wedge-shaped path depicted in Fig.~\ref{fig:1D_tree}~(a).


\subsection{Shortest Path Formulation}
The utility of path-planning algorithms is made non-trivial by the introduction of obstacles in the path space. Let $X_{obs}\subset \mathbb{R}^d$ be a closed subset of spatial points representing obstacles. The initial configuration of the robot is defined as $z_{\text{init}}=(x_0, P_0)\in \mathbb{R}^d\times \mathbb{S}_{++}^d$, while $\mathcal{Z}_{\text{target}} \subset \mathbb{R}^d\times \mathbb{S}_{++}^d$ is a given closed subset representing the target region which the robot desires to attain. Given a confidence level parameter $\chi^2>0$, the shortest RI path problem can be formulated as:
\begin{equation}
\label{eq:main_problem}
\begin{split}
    \min_{\gamma \in \mathcal{BV}[0, T]} \;\; & c(\gamma) \\
    \text{ s.t. }\;\;\;\; & \gamma(0)=z_{\text{init}}, \; \; \gamma(T)\in \mathcal{Z}_{\text{target}} \\
    & (x(t)-x_{\text{obs}})^\top P^{-1}(t)(x(t)-x_{\text{obs}}) \geq \chi^2,\\
    &\qquad \forall t\in [0, T], \;\; \forall x_{\text{obs}}\in X_{obs}.
\end{split}
\end{equation}
The $\chi^2$ term in the constraints of (\ref{eq:main_problem}) is implemented to provide a confidence bound on probability that a robot with position $x(t)$ will not be in contact with an obstacle $x_{\text{obs}}$.

\section{RI-RRT* Algorithm}
\label{sec:RRT_main}

\subsection{RRT*}
\label{sec:RRT}
The RRT algorithm \cite{Lavalle2001} constructs a tree of nodes (state realizations) through random sampling of the feasible state-space and then connects these nodes with edges (tree branches). A user-defined cost is utilized to quantify the length of the edges, which are in turn summed to form path lengths. Each new node is connected via a permanent edge to the existing node which provides the shortest path between the new node and the initial node of the tree. 
Although the RRT algorithm is known to be probabilistically complete (the algorithm finds a feasible path if one exists), it does not achieve asymptotic optimality (path cost does not converge to the optimal one as the number of nodes is increased)
 \cite{Karaman2010}. The RRT* algorithm \cite{Karaman2010} attains asymptotic optimality by including an additional ``re-wiring" step that re-evaluates if the path length for each node can be reduced via a connection to the newly created node. This paper utilizes RRT* as a numerical approach to the shortest path problem (\ref{eq:main_problem}).

\subsection{Algorithm}
\label{sec:IGPP_alg}
Provided below is a Rationally Inattentive RRT* (RI-RRT*) algorithm for finding a solution to (\ref{eq:main_problem}). Like the original RRT* algorithm, the RI-RRT* algorithm constructs a graph of state nodes and edges ($G\leftarrow (Z,E)$) in spaces with or without obstacles.
\begin{algorithm}[h]\label{alg:igpp_body1}
    \footnotesize{
    $(z_1) \leftarrow (z_{\text{init}})$; $E \leftarrow \emptyset$; $G'\leftarrow (z_1,E)$\;
    \For{$i = 2:N$}{
        $G \leftarrow G'$ \;
        $z_i = (x_i,P_i) \leftarrow \textsc{\fontfamily{cmss}\selectfont Generate}(i)$\;
        $(Z',E') \leftarrow (Z,E)$\;
        $z_{\text{near}} \leftarrow \textsc{\fontfamily{cmss}\selectfont Nearest}(Z',z_i)$\;
        $z_{\text{new}} \leftarrow \textsc{\fontfamily{cmss}\selectfont Scale} (z_{\text{near}},z_i,ED_{\text{min}})$\;
        \If{$\textsc{\fontfamily{cmss}\selectfont ObsCheck} (z_{\text{near}},z_{\text{new}}) = False$}{
            $Z' \leftarrow Z' \cup z_{\text{new}}$\;
            $Z_{\text{nbors}} \leftarrow \textsc{\fontfamily{cmss}\selectfont Neighbor} (Z,z_{\text{new}},ED_{\text{nbors}})$\;
            $Path_{z_{\text{new}}} \leftarrow realmax$\;
            \For{$z_j \in Z_{\text{nbors}}$}{
                \If{$\textsc{\fontfamily{cmss}\selectfont ObsCheck} (z_j,z_{\text{new}}) = False$}{
                    $Path_{z_{\text{new},j}} \leftarrow Path_{z_j} + \mathcal{D}(z_j,z_{\text{new}})$\;
                    \If{$Path_{z_{\text{new},j}} < Path_{z_{\text{new}}}$}{
                        $Path_{z_{\text{new}}} \leftarrow Path_{z_{\text{new},j}} $\;
                        $z_{\text{nbor}}^* \leftarrow z_j$\;
                    }
                }
            }
            $E' \leftarrow \left[z_{\text{nbor}}^*,z_{\text{new}} \right] \cup E'$\;
            \For{$z_j \in Z_{\text{nbors}} \: \backslash \: z_{\text{nbor}}^*$}{
                \If{$\textsc{\fontfamily{cmss}\selectfont ObsCheck} (z_{\text{new}},z_j) = False$}{
                    $Path_{z_j,\text{rewire}}=Path_{z_{\text{new}}} + \mathcal{D}(z_{\text{new}},z_j)$\;
                    \If{$Path_{z_j,\text{rewire}} < Path_{z_j}$}{
                        $E' \leftarrow E' \cup \left[ z_{\text{new}},z_j \right] \backslash \left[ z_{j,\text{parent}}, z_j \right]$\;
                        $z_{j,\text{parent}} \leftarrow z_{\text{new}}$\;
                        $\textsc{\fontfamily{cmss}\selectfont UpdateDes}(G, z_j)$\;
                    }            
                }
            }
        }$G' \leftarrow (Z',E')$
    }
    }
\caption{RI-RRT* Algorithm}
\end{algorithm}

In Algorithm \ref{alg:igpp_body1}, the $\textsc{\fontfamily{cmss}\selectfont Generate}(i)$ function creates a new point by randomly sampling a spatial location ($x \in \mathbb{R}^d$) and covariance ($P\in \mathbb{S}_{++}^d$). Notice that for a $d$-dimensional configuration space, the corresponding uncertain configuration space $\mathbb{R}^d \times \mathbb{S}^d_{++}$ has $d+\frac{1}{2}d(d+1)$ dimensions from which the samples are generated. The $\textsc{\fontfamily{cmss}\selectfont Nearest}(Z ,z_i)$ function finds the nearest point ($z_\text{near}$), in metric $\hat{\mathcal{D}}(z,z'):=\|x-x'\|+\|P-P'\|_{F}$, between the newly generated state $z_i=(x_i,P_i)$ and an existing state in the set $Z$.
Using the metric $\hat{\mathcal{D}}(z,z')$, the $\textsc{\fontfamily{cmss}\selectfont Scale}(z_{\text{near}},z_i,ED_{\text{min}})$ function linearly shifts the generated point ($z_i$) to a new location as:
\begin{equation*} 
z_{\text{new}} \!=\!\!
\begin{cases}
    z_{\text{near}} \!+\! \frac{ED_{\text{min}}}{\hat{\mathcal{D}}(z_i,z_\text{near})}\left(z_i-z_{\text{near}}\right)~ \text{if}~\hat{\mathcal{D}}(z,z') > ED_\text{min}, \\
    z_i \hspace{3.8cm} \text{otherwise,}
\end{cases}
\end{equation*}
where $ED_\text{min}$ is a user-defined constant.
In addition to generating $z_{\text{new}}$, the $\textsc{\fontfamily{cmss}\selectfont Scale}$ function also ensures that its $\chi^2$ covariance region does not interfere with any obstacles.

\begin{figure}[t]
\centering
\includegraphics[width = 0.6\columnwidth]{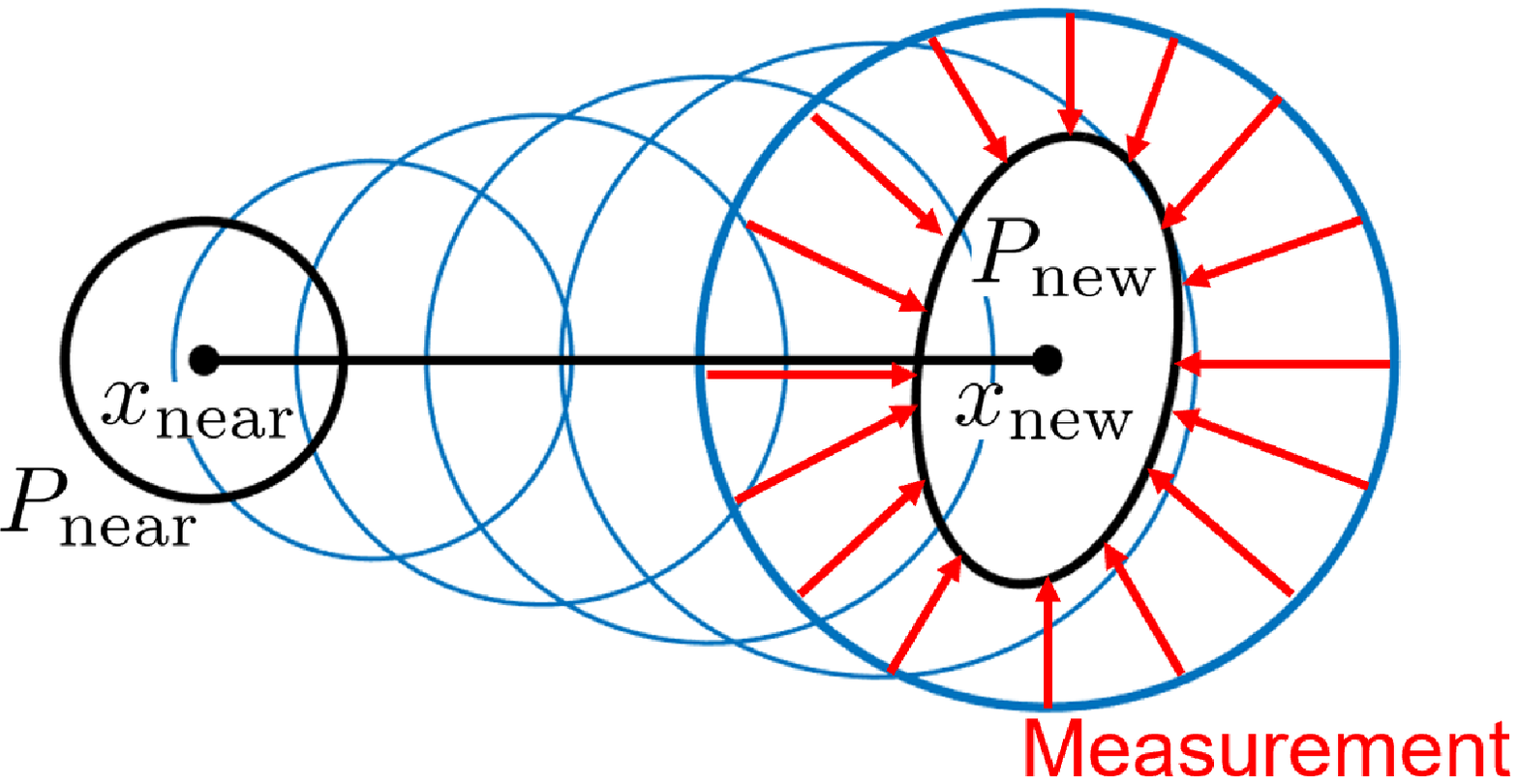}
\caption{Transition from state $z_\text{near}$ to $z_\text{new}$. The blue ellipses represent the propagation of covariance where the $\textsc{\fontfamily{cmss}\selectfont ObsCheck}(z_{\text{near}},z_{\text{new}})$ function in Algorithm \ref{alg:igpp_body1} checks collisions between these propagated covariances, including $z_\text{near}$ and $z_\text{new}$, and obstacles.
The measurement at $z_\text{new} = (x_\text{new}, P_\text{new})$ makes a covariance smaller.} 
\label{fig:prop_cov}
\end{figure}

The $\textsc{\fontfamily{cmss}\selectfont ObsCheck}(z_{\text{near}},z_{\text{new}})$ function ensures that transition from state $z_{\text{near}}$ to $z_{\text{new}}$ does not intersect with the obstacles. More precisely, we assume the transition $z_\text{near} \rightarrow z_\text{new}$ follows the move-and-sense path, introduced in Section~\ref{sec:Prelim-D}, and the $\textsc{\fontfamily{cmss}\selectfont ObsCheck}$ function returns $False$ if all state pairs along the move-and-sense path, shown by blue ellipses in Fig. \ref{fig:prop_cov}, has $\chi^2$ covariance regions that are non-interfering with obstacles.


The $\textsc{\fontfamily{cmss}\selectfont Neighbor}(Z, z_{\text{new}}, ED_{\text{nbors}})$ function returns the sub-set of nodes described as $ Z_{\text{nbors}} =  \{ z_i=(x_i,P_i) \in Z: \hat{\mathcal{D}}(z_i,z_\text{new}) \leq ED_{\text{nbors}} \}$. This set is then evaluated for the presence of obstacles via the $\textsc{\fontfamily{cmss}\selectfont ObsCheck}$ function of the previous paragraph. Note that in this instance, the function is evaluating obstacle interference along the continuous path of state-covariance pairs from $z_j \rightarrow z_{\text{new}}$. Lines 14-17 of Algorithm \ref{alg:igpp_body1} connect the new node to the existing graph in an identical manner to RRT*, where the $Path_{z}$ denote the cost of the path from the $z_{init}$ to node $z$ through the edges of $G$.
Line 18 creates a new edge between the new node and the existing nodes from the neighbor group $Z_{\text{nbors}}$ which results in the minimum $Path_{z_{\text{new}}}$.
The calculation of RI path cost in Line 14 utilizes (\ref{eq:def_D}).

Lines 19-24 are the tree re-wiring steps of Algorithm \ref{alg:igpp_body1}. In line 20, the $\textsc{\fontfamily{cmss}\selectfont ObsCheck}$ function is called again. This is because the move-and-sense path is direction-dependent, and thus $\textsc{\fontfamily{cmss}\selectfont ObsCheck}(z_j, z_{\text{new}}) = False$ does not necessarily imply  $\textsc{\fontfamily{cmss}\selectfont ObsCheck}(z_{\text{new}}, z_j) = False$. Finally, for each rewired node $z_j$, its cost (i.e., $Path_{z_j}$) and the cost of its descendants are updated via $\textsc{\fontfamily{cmss}\selectfont UpdateDes}(G, z_j)$ function in line 25.


To increase the  computational efficiency of RI-RRT* algorithm we deploy a branch-and-bound technique as detailed in \cite{karaman2011anytime}. For a  given tree $G$, let $z_\text{min} $ be the node that has the lowest cost along the nodes of $G$ within $\mathcal{Z}_{\text{target}}$. As discussed in Section~\ref{sec:Prelim-D}, $\mathcal{D}(z,z_\text{goal})$ is a lower-bound for the cost of transitioning from $z$ to $z_\text{goal}$. The branch-and-bound algorithm periodically deletes the nodes $Z'' = \{z\in Z: Path_z +\mathcal{D}(z,z_\text{goal}) \geq Path_{z_\text{min}}\}$. This elimination of the non-optimal nodes speeds up the RI-RRT* algorithm.

\subsection{Properties of RI-RRT*}
\label{sec:Alg_properties}
The question regarding the asymptotic optimality of RI-RRT* naturally arises. Recall that the proof of the asymptotic optimality of the RRT* algorithm \cite{Karaman2011} is founded on four main assumptions:
\begin{enumerate}
    \item additivity of the cost function, 
    \item the cost function is monotonic,
    \item there exists a finite distance between all points on the optimal path and the obstacle space,
    \item the cost function is Lipschitz continuous, either in the topology of total variation metric \cite{Karaman2011} or the supremum norm metric \cite{Karaman2010}.
\end{enumerate}
The proofs of the first three assumptions are trivial for the RI cost (\ref{eq:def_D}). However, Theorem \ref{theo:continuity} does not suffice to guarantee that the RI cost meets the fourth condition for $d\geq 2$. For this reason, currently the asymptotic optimality of the RI-RRT* algorithm cannot be guaranteed, while the numerical simulations of Section \ref{sec:Results} do show that the proposed algorithm does have merit in rationally inattentive path-planning.
   
\section{Simulation Results}\label{sec:Results}

\subsection{One-Dimensional Simulation}\label{sec:1D}

\begin{figure}[t]
\centering
\centering
    \subfloat[A generated path]
    {\includegraphics[trim = 0cm 0cm 0.2cm 0.5cm, clip=true, width=0.4\columnwidth]{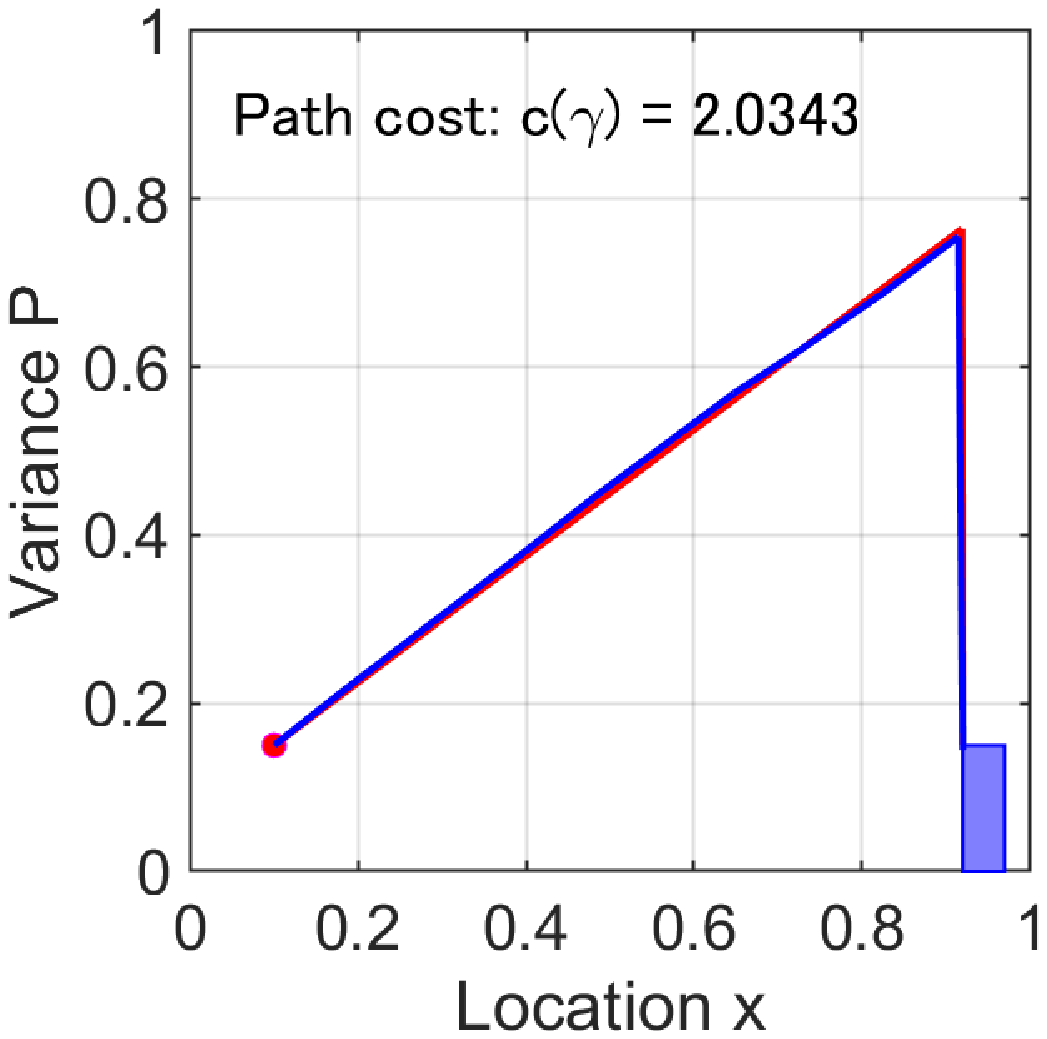}
    \label{fig:exp_snap_init}} \quad
    \subfloat[Path cost for $100$ runs]
    {\includegraphics[trim = 0cm 0cm 0.2cm 0.5cm, clip=true, width=0.45\columnwidth]{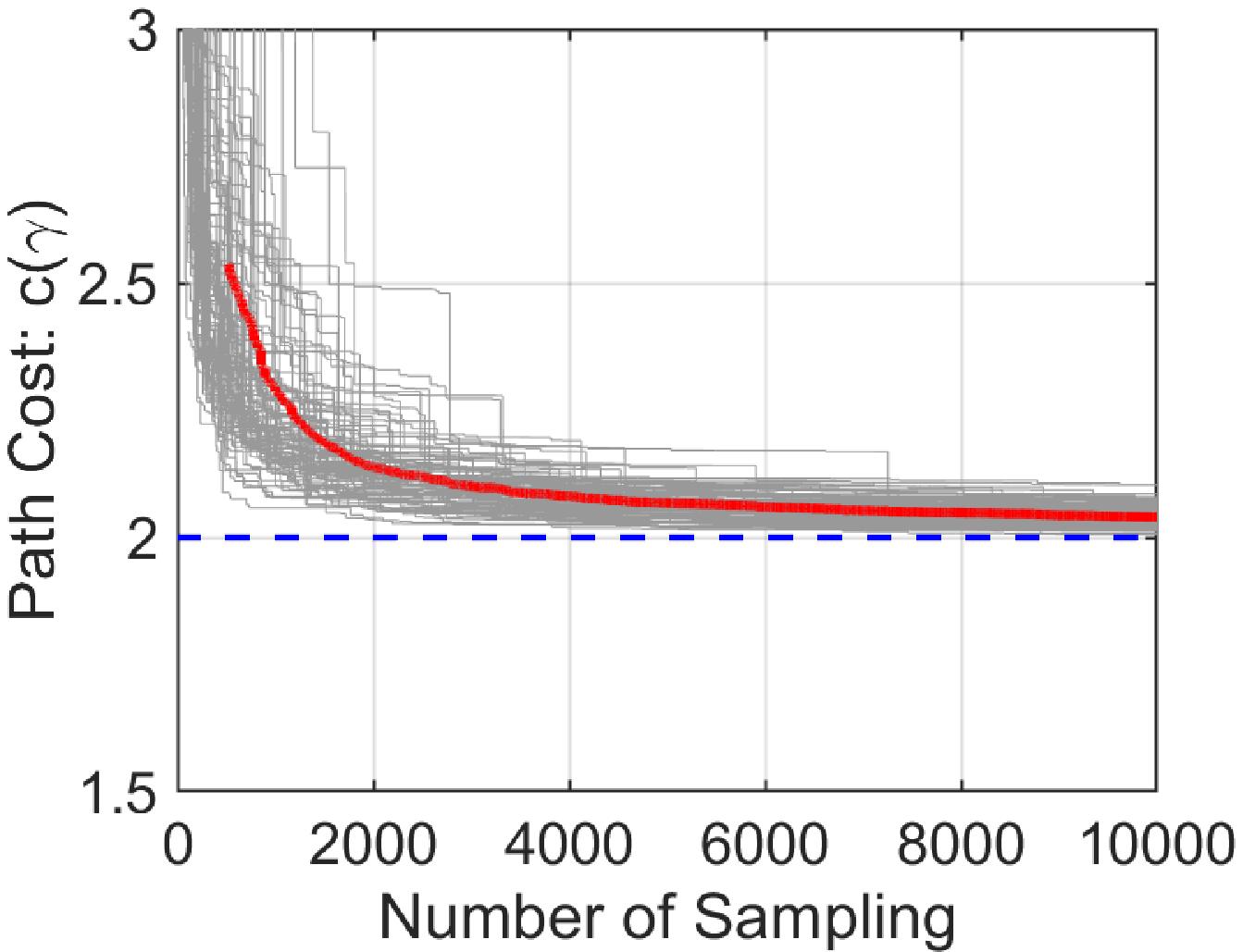}
    \label{fig:exp_snap_fin}}
\caption{Results of the RI-RRT* algorithm with $\alpha = 1$ and $W=0.75$ applied to one-dimensional joint movement-perception problem. (a): The blue line illustrates the path generated with $10,000$ nodes, which almost converges to the known optimal path depicted as the red curve. (b): The total path cost for $100$ runs of $10,000$ nodes. Gray lines plot the path cost for each run, while the average of them is shown in the red line. The path costs of all $100$ runs approach the optimal path cost (dashed blue line).}
\label{fig:1D_tree}
\end{figure}
The first study is the case of a robot which is allowed to travel at a constant velocity in a one-dimensional geometric space from a predetermined initial position and covariance $z_0=(x_0,P_0)$, specified by the red dot in Fig.~\ref{fig:1D_tree}~(a). The robot has a goal of reaching some final state within the blue box representing a goal region which is a sub-set of the reachable space.  
Note that the goal region contains acceptable bounds on both location and uncertainty. 
Although, in one-dimensional setting, the strategy which minimizes the control cost is obviously the one that moves directly toward the target region, we utilize the RI-RRT* algorithm
to solve the non-trivial measurement scheduling problem.

In Fig.~\ref{fig:1D_tree}~(a), the blue curve represents the path generated by the RI-RRT* algorithm with $10,000$ nodes, which is sufficiently close to the shortest path obtainable via the RI-distance depicted as the red curve. 
These wedge-shaped optimal paths are created by the ``move-and-sense" strategy integrated in the RI-cost, where it has a section of covariance propagation followed by an instantaneous reduction of covariance as discussed in Section \ref{sec:Prelim-D}.
For example, if the robot were an autonomous ground vehicle with GPS capabilities, then this path signifies the robot driving the total distance without any GPS updates, followed by a reduction its spatial uncertainty with a single update once the goal region is reached. 
The minimum path cost at the end of each iteration of the $\textit{for-loop}$ in Algorithm \ref{alg:igpp_body1} is depicted in Fig.~\ref{fig:1D_tree}~(b), where the red curve represents the average of $100$ independent simulations.
The path cost of each simulation approaches to the optimal cost.

\subsection{Two-Dimensional Asymmetric Simulation}\label{sec:2D_asym}
\begin{figure}[t]
\centering
{\includegraphics[trim = 0cm 0cm 0cm 0cm, clip=true, width=0.9\columnwidth]{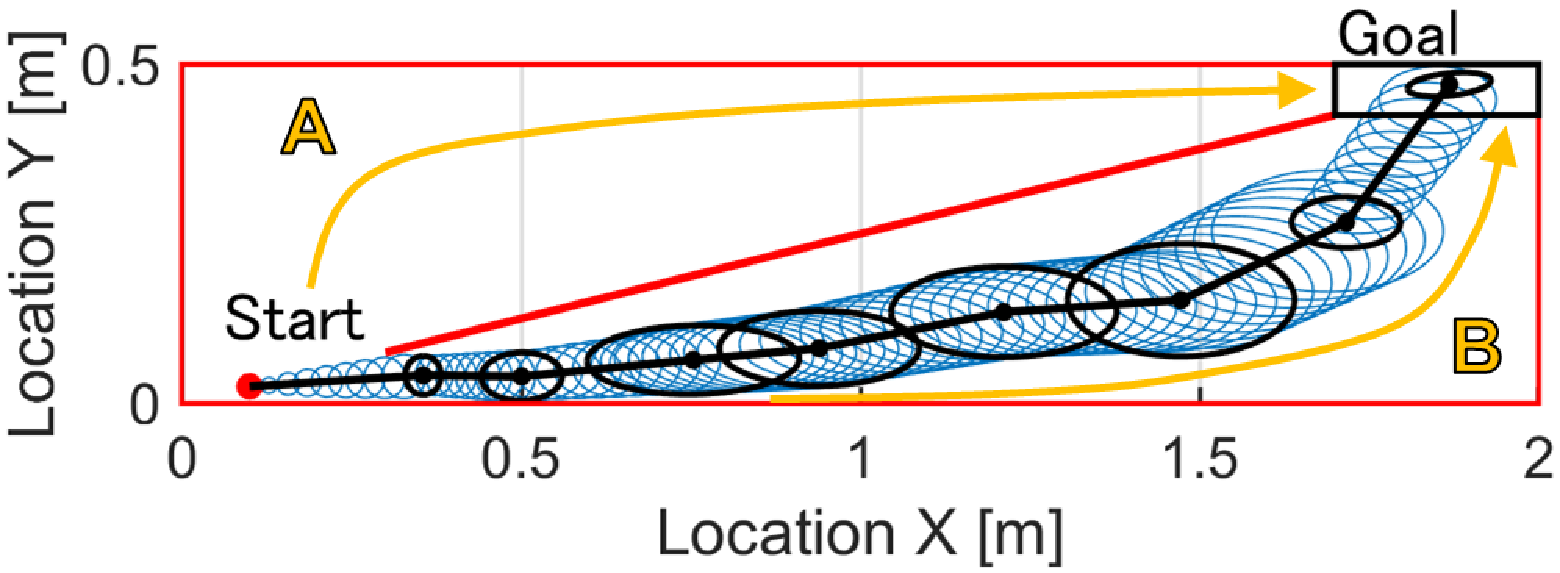}}
\caption{Results of the RI-RRT* algorithm with $10,000$ nodes in the two-dimensional space containing roughly two paths, A and B, separated by a diagonal wall. The black line is the shortest path with the associated covariance ellipses. The blue ellipses illustrate the propagation of covariance between nodes. The simulation was completed with $W = 10^{-3} I$ and $\chi^2$ covariance ellipses representing $90 \%$ certainty regions. The boundaries of the plots are considered as obstacles.}
\label{fig:2D_asym}
\end{figure}

The asymmetric characteristic of the RI-cost is demonstrated via a simulation in the two-dimensional configuration space with a diagonal wall, as seen in Fig.~\ref{fig:2D_asym}.
The initial position and covariance of the robot is depicted as a red dot, 
while the target region is illustrated as the black rectangle at the upper-right corner.

The path is generated by the RI-RRT* algorithm by sampling $10,000$ nodes. The corresponding sampled covariance ellipses are shown in black where the blue ellipses represent covariance propagation. 
As shown in Fig.~\ref{fig:2D_asym}, there are two options; path A requires the robot to move into a funnel-shaped corridor, while the path B moves out of a funnel.
In this setting, the RI-cost prefers the path B even though both A and B have the same Euclidean distance. This asymmetric behavior results from the fact that as the robot approaches the goal region path B requires a less severe uncertainty 
reduction compared to path A.
Similarly, by exchanging the start and goal positions, the RI-cost prefers path A over path B, thus displaying the directional dependency of our efficient sensing strategy.



\subsection{Two-Dimensional Simulation with Multiple Obstacles}\label{sec:2D_mult}
\begin{figure*}[t!]
    \centering
    \subfloat[$\alpha = 0$]
    {\includegraphics[trim = 0.3cm 0cm 1.5cm 0.82cm, clip=true, width=0.63\columnwidth]{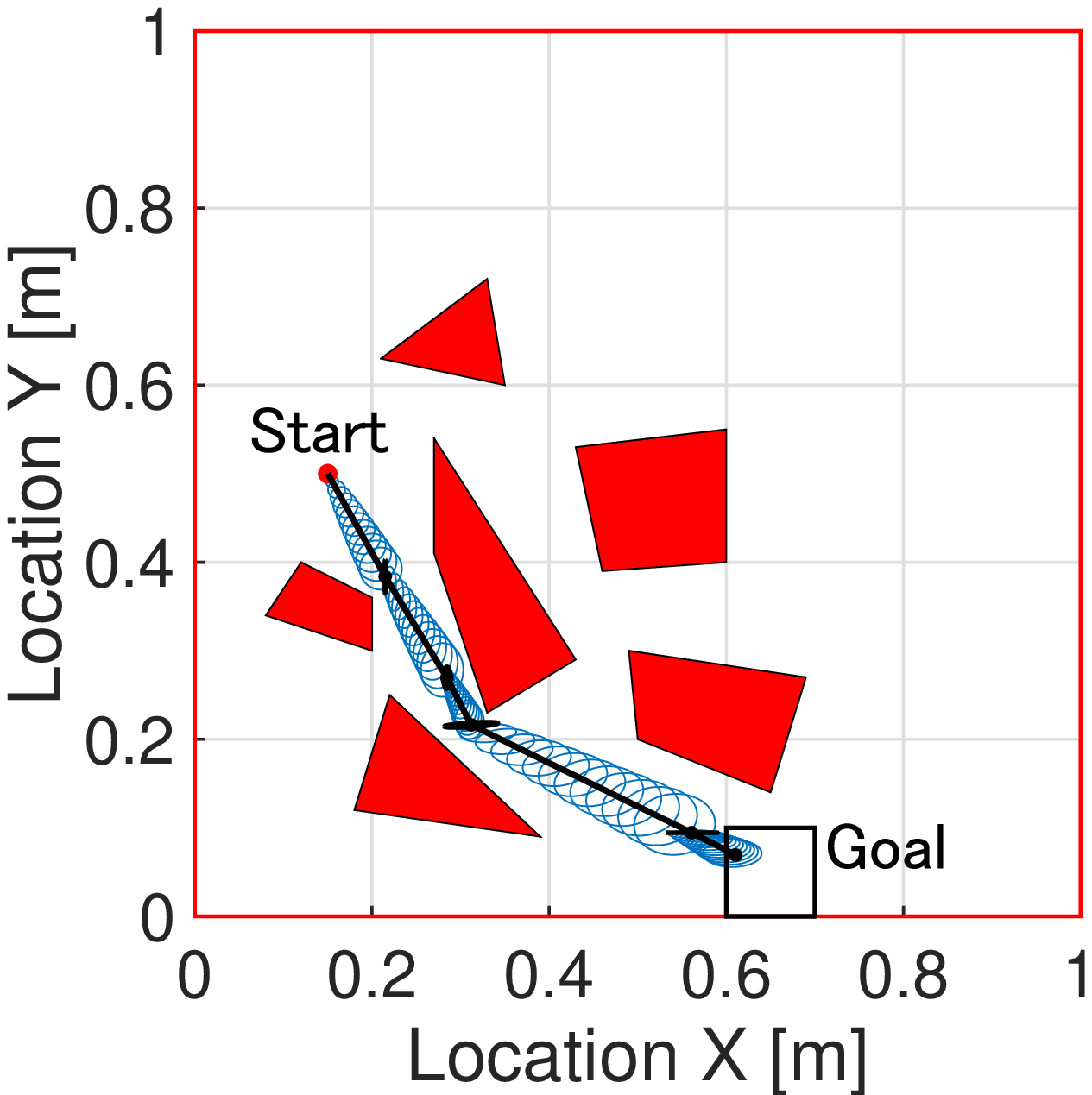}
    \label{fig:sim_snap_nocbf}} \quad 
    \subfloat[$\alpha = 0.1$]
    {\includegraphics[trim = 0.3cm 0cm 1.5cm 0.82cm, clip=true, width=0.63\columnwidth]{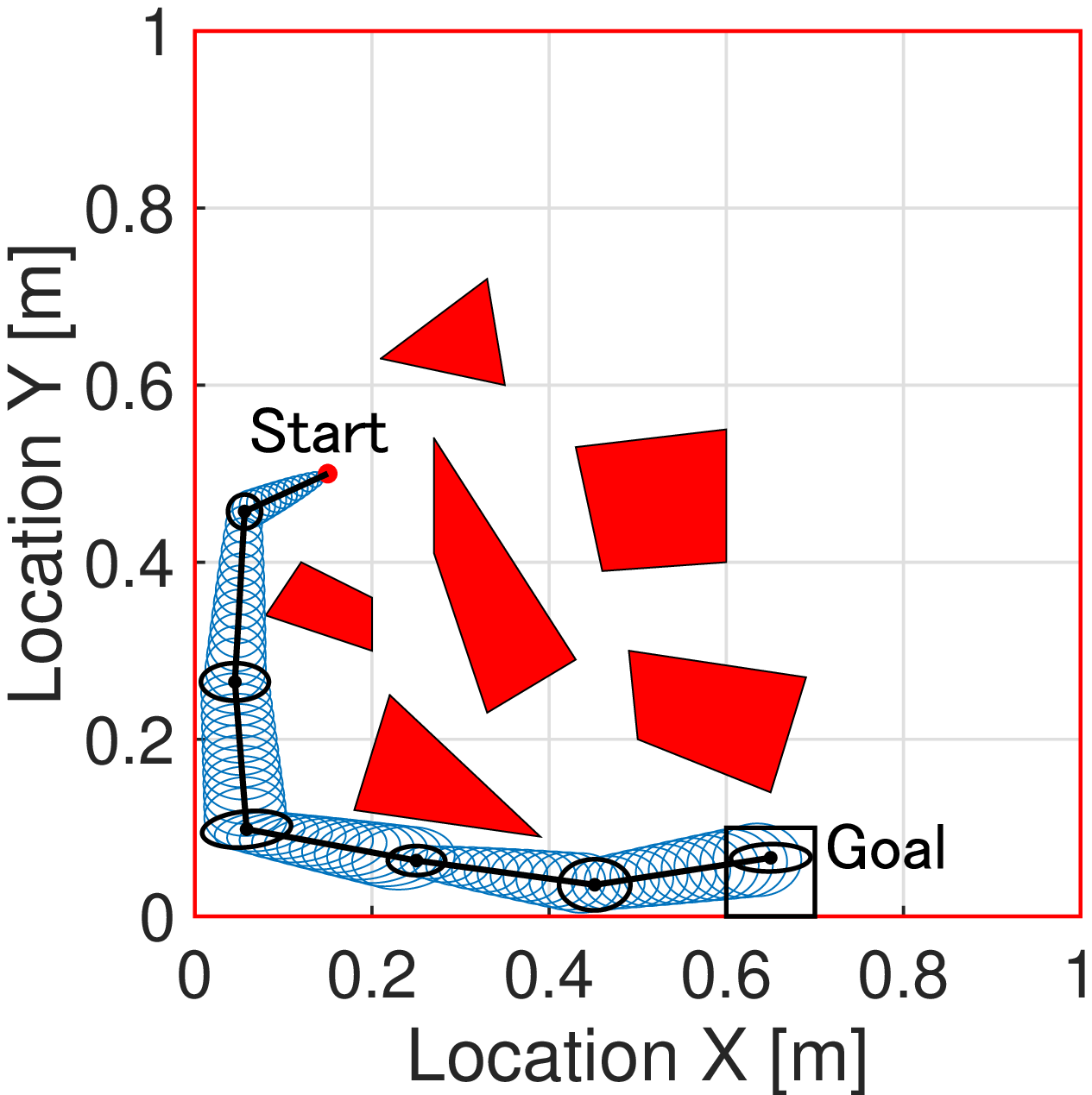}
    \label{fig:sim_snap_cbf19}} \quad
    \subfloat[$\alpha=0.3$]
    {\includegraphics[trim = 0.3cm 0cm 1.5cm 0.82cm, clip=true, width=0.63\columnwidth]{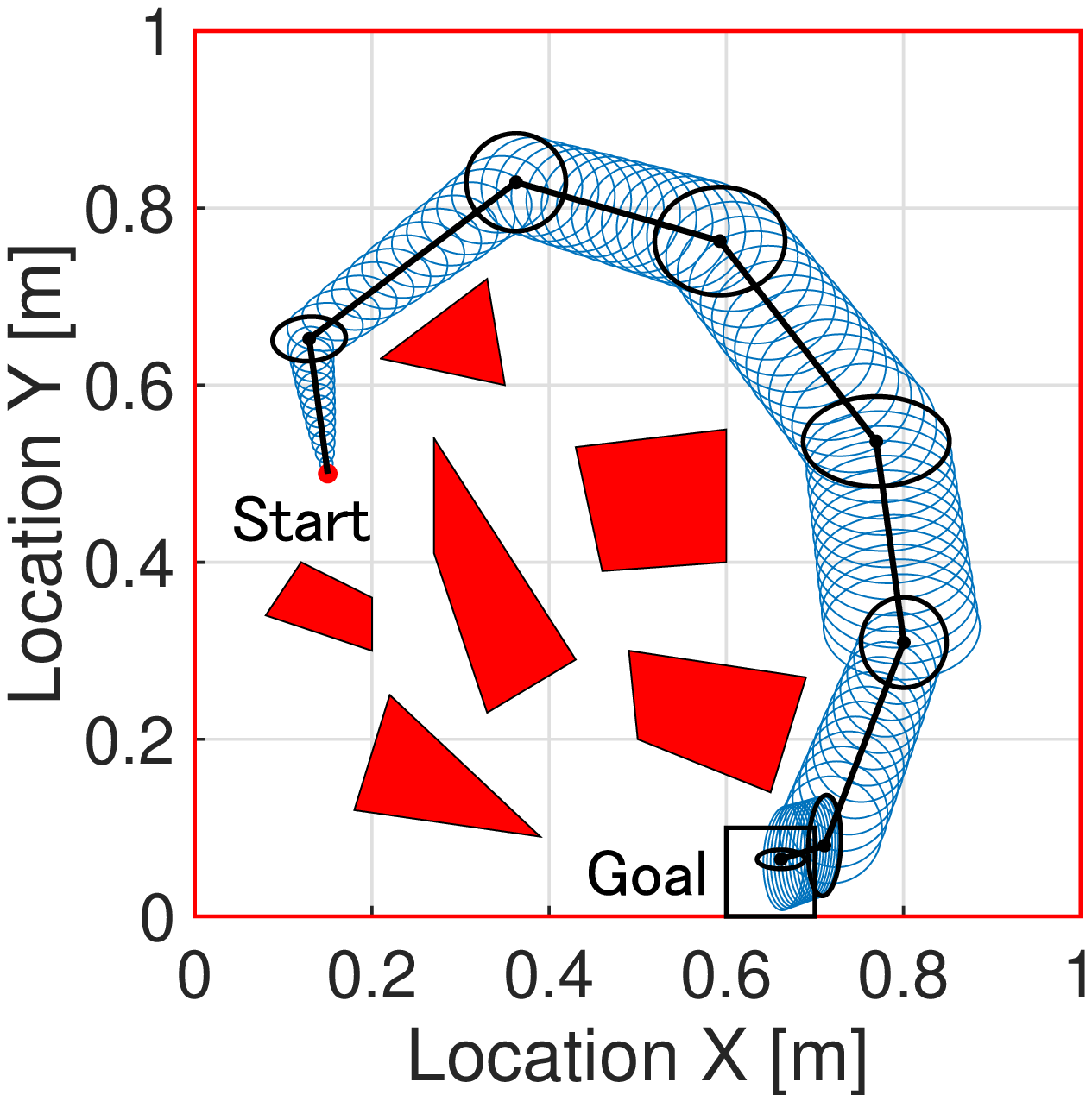}
    \label{fig:sim_snap_cbf20}}
    
    \caption{Results of simulation with $\alpha = 0, 0.1, 0.3$ under the existence of multiple obstacles. The simulation was completed with $W = 10^{-3} I$ and $\chi^2$ covariance ellipses representing $90 \%$ certainty regions. The boundaries of the plots are considered as obstacles.}
    \label{fig:2D_mult}
\end{figure*}


In a final demonstration, the RI-RRT* algorithms with $\alpha = 0, 0.1,$ and $0.3$ are implemented in a two-dimensional configuration space containing multiple obstacles in order to illustrate the effects of varying the information cost. All three paths in panels of Fig.~\ref{fig:2D_mult} are generated from 4,000 nodes.

As seen in Fig.~\ref{fig:2D_mult}~(a), when $\alpha = 0$ the algorithm simply finds a path that has the shortest Euclidean distance, 
even if that path requires frequent sensing actions. 
In contrast, the RI-RRT* algorithm with $\alpha = 0.3$ does not take a constrained pathway, and thus requires fewer sensor actuations in order to avoid obstacle collisions. As a result, the algorithm deviates from the shortest Euclidean distance path and allows the covariance to propagate safely, as seen in Fig.~\ref{fig:2D_mult}~(c).
A moderate path, illustrated in Fig.~\ref{fig:2D_mult}~(b), can be also obtained by choosing $\alpha = 0.1$.

\section{Conclusion}
\label{sec:Conclusion}
In this work, a novel RI cost for utilization in path-planning algorithms is presented. The cost accounts for both the path distance traversed (efficiency) and the amount of information which must be perceived by the robot. Information gained from perception is important in that it allows the robot to safely navigate obstacle-filled environments with confidence that collisions will be avoided. This method, in balancing path distance and perception costs, provides a simplicity-based path which can be tailored to mimic the results potentially generated by an expert human path-planner. Three numerical simulations were provided demonstrating these results.


Currently, the preliminary version of the RI-RRT* algorithm is optimized for computational efficiency by the aid of a \emph{branch-and-bound} technique. The authors note that utilizing other RRT* improvement methods, such as \emph{k-d trees} could further improve the computational speed of Algorithm \ref{alg:igpp_body1} and should be considered in future work.
In the same vein of future work, the authors note the importance of quantifying the impact that the user-defined constants, such as the distances which signify which nodes are neighbors, have on the results of the RI-RRT* algorithm. Also, the topic of path refinement should be further explored as RRT*-like algorithms converge asymptotically. 

It should be noted that once the RI-RRT* algorithm finds an initial feasible path, there exist iterative methods for path ``smoothing" which do not require additional node sampling. In a theorized hybrid method, RI-RRT* is first utilized to find some initial path and then an iterative method takes the initial path and ``smooths" towards the optimal path. Convergence benefits of iterative methods are often improved as the initial path guess is more similar to the optimal path, and a trade-off could be found in computational efficiency which results in the best time to switch between  algorithms.

\begin{appendices}

\section{Proof of Lemma \ref{lem:ineq}} \label{sec:AppenI}

For every $ t \in [0, T]$, set a partition $\mathcal{P}=(0, t, T)$. Then
\begin{align*}
    \|x(t)\|+\bar{\sigma}(P(t)) &\leq \|x(0)\|+\bar{\sigma}(P(0)) \nonumber \\
    & \qquad +  \|x(t)-x(0)\|+\bar{\sigma}(P(t)-P(0)) \nonumber \\
    &\leq \|x(0)\|+\bar{\sigma}(P(0)) \nonumber \\
    & \qquad +\|x(t)-x(0)\|+\bar{\sigma}(P(t)-P(0)) \nonumber \\
    & \qquad + \|x(T)-x(t)\|+\bar{\sigma}(P(T)-P(t)) \nonumber \\
    &\leq V(\gamma, \mathcal{P}) \leq |\gamma|_{\text{TV}}.
\end{align*}

\section{Proof of Theorem \ref{theo:continuity}}\label{sec:AppenII}
The proof is based on the following lemma:
\begin{lemma}
\label{lem:lipschitz}
Assume $d=1$. For each $(\epsilon,\delta)$ satisfying $0<\delta\leq\frac{\epsilon}{2}$, there exists a constant $L_\epsilon$ such that the inequality:
\begin{equation}
\begin{split}
    & |\mathcal{D}(x'_k, x'_{k+1}, P'_k, P'_{k+1})-\mathcal{D}(x_k, x_{k+1}, P_k, P_{k+1})| \\
    & \qquad \leq L_\epsilon \Bigl[ \left|(x'_{k+1}-x_{k+1})-(x'_k-x_k)\right|  \\ 
    & \qquad \qquad \qquad + \left|(P'_{k+1}-P_{k+1})-(P'_k-P_k)\right| \\
    & \qquad \qquad \qquad + \delta \left|P_{k+1}-P_{k}\right|+ \delta\left|x_{k+1}-x_k\right|\Bigr] \nonumber \\
\end{split}
\end{equation}
holds for all
\begin{equation}
    \begin{split}
        & x_k',x_{k+1}',x_k,x_{k+1} \in \mathbb{R} \; \; \text{and} \; \; P_k',P_{k+1}',P_k,P_{k+1} \geq \epsilon \nonumber
    \end{split}
\end{equation}
such that
\begin{equation}
    \begin{split}
        & \Delta x_k := x_k' - x_k \leq \delta, \; \; \Delta x_{k+1} := x_{k+1}' - x_{k+1} \leq \delta \\
        & \Delta P_k := P_k' - P_k \leq \delta, \; \; \Delta P_{k+1} := P_{k+1}' - P_{k+1} \leq \delta. \nonumber
    \end{split}
\end{equation}
\end{lemma}
\begin{proof}
    For simplicity, we assume $\alpha=W=1$, but the extension of the following proof to general cases is straightforward. In what follows, we write
\begin{align*}
&\mathcal{D}_{\text{info}}(x_k,x_{k+1},P_k,P_{k+1})\\
&:=\max \left\{0,\frac{1}{2}\log_2(P_k+\left|x_{k+1}-x_k\right|) - \frac{1}{2}\log_2(P_{k+1})\right\}.
\end{align*}
    We consider four different cases depending on the signs of $\mathcal{D}_{\text{info}}(x_k',x_{k+1}',P_k',P_{k+1}')$ and $\mathcal{D}_{\text{info}}(x_k,x_{k+1},P_k,P_{k+1})$.
    
    \underline{Case 1}: First, we consider the case with
    \begin{align*}
    &\mathcal{D}_{\text{info}}(x_k',x_{k+1}',P_k',P_{k+1}')>0 \text{ and } \\ 
    &\mathcal{D}_{\text{info}}(x_k,x_{k+1},P_k,P_{k+1})>0.
    \end{align*}
In this case:
    \begin{equation} \label{eq:one} \small
        \begin{split} 
            & \left|\mathcal{D}(x'_k, x'_{k+1}, P'_k, P'_{k+1})-\mathcal{D}(x_k, x_{k+1}, P_k, P_{k+1})\right| \\
            & =\bigg| \left|x_{k+1}+\Delta x_{k+1}-x_k-\Delta x_k\right|-\left|x_{k+1}-x_k\right| \\
            & \quad + \left. \frac{1}{2}\log_2(P_k+\Delta P_k + \left|x_{k+1}+\Delta x_{k+1}-x_k-\Delta x_k\right|) \right. \\
            & \quad \left. - \frac{1}{2}\log_2(P_{k+1}+\Delta P_{k+1}) - \frac{1}{2}\log_2(P_k+\left|x_{k+1}-x_k\right|) \right. \\
            & \quad +  \frac{1}{2}\log_2(P_{k+1}) \bigg| \\
            &  \leq \bigg| \left|x_{k+1}+\Delta x_{k+1}-x_k-\Delta x_k\right|-\left|x_{k+1}-x_k\right| \bigg| \\
            & \quad + \frac{1}{2}\left| \log_2\left(1 + \frac{\Delta P_k}{P_k} + \frac{\left|x_{k+1}+\Delta x_{k+1}-x_k-\Delta x_k\right|}{P_k} \right) \right.\\
            &\left. \quad - \log_2\left( 1 + \frac{\left|x_{k+1}-x_k\right|}{P_k} + \frac{P_k+\left|x_{k+1}-x_k\right|}{P_k P_{k+1}} \Delta P_{k+1} \right) \right|.
        \end{split}
    \end{equation}
    Using the fact that $\left| |a+b|-|a| \right| \leq \left| b \right|$ for $a,b\in\mathbb{R}$, we have:
    \begin{equation}
    \begin{split}
        &\left| \left|x_{k+1}+\Delta x_{k+1}-x_k-\Delta x_k\right|-\left|x_{k+1}-x_k\right| \right| \\
        & \qquad \leq \left| \Delta x_{k+1}- \Delta x_k \right|. \nonumber
    \end{split}
    \end{equation}
    Noticing that the arguments in the logarithmic terms in (\ref{eq:one}) are $\geq \frac{1}{2}$, and using the fact that $\left| \log_2(a)-\log_2(b)\right| \leq 2 \log_2(e)|a-b|, \; \forall a,b\geq \frac{1}{2}$, we have:
    \begin{equation} 
        \begin{split}
            & \frac{1}{2}\left| \log_2\left(1 + \frac{\Delta P_k}{P_k} + \frac{\left|x_{k+1}+\Delta x_{k+1}-x_k-\Delta x_k\right|}{P_k} \right) \right.\\
            &\left. \quad - \log_2\left( 1 + \frac{\left|x_{k+1}-x_k\right|}{P_k} + \frac{P_k+\left|x_{k+1}-x_k\right|}{P_k P_{k+1}} \Delta P_{k+1} \right) \right|\\
            & \leq \log_2(e) \left| \frac{\left|x_{k+1}+\Delta x_{k+1}-x_k-\Delta x_k\right|}{P_k} - \frac{\left|x_{k+1}-x_k\right|}{P_k} \right.\\
            &\left. \quad + \frac{\Delta P_k}{P_k} - \frac{P_k+\left|x_{k+1}-x_k\right|}{P_k P_{k+1}}\Delta P_{k+1} \right| \\
            & \leq \frac{\log_2(e)}{P_k}\bigg| |x_{k+1}+\Delta x_{k+1}-x_k-\Delta x_k| - |x_{k+1}-x_k| \bigg| \\
            & \quad +\! \log_2(e) \bigg| \frac{\Delta P_k \!-\! \Delta P_{k+1}}{P_k} \\ &\hspace{2cm}+\frac{P_{k+1}\!-\!P_k\!-\!|x_{k+1}\!-\!x_k|}{P_k P_{k+1}} \Delta P_{k+1}\bigg| \\
            & \leq \frac{\log_2(e)}{P_k} | \Delta x_{k+1} - \Delta x_k | + \frac{\log_2(e)}{P_k} | \Delta P_{k+1} - \Delta P_k | \\
            & \quad + \frac{\log_2(e)\left| \Delta P_{k+1} \right|}{P_k P_{k+1}} \big| P_{k+1}-P_k - |x_{k+1}-x_k| \big| \\
            & \leq \frac{\log_2(e)}{P_k} | \Delta x_{k+1} - \Delta x_k | + \frac{\log_2(e)}{P_k} | \Delta P_{k+1} - \Delta P_k | \\
            & \quad + \frac{\log_2(e)|\Delta P_{k+1}|}{P_k P_{k+1}} | x_{k+1}-x_k | \\
            & \quad + \frac{\log_2(e)|\Delta P_{k+1}|}{P_k P_{k+1}} | P_{k+1}-P_k | \\
            & \leq \frac{\log_2(e)}{\epsilon} | \Delta x_{k+1} - \Delta x_k | + \frac{\log_2(e)}{\epsilon} | \Delta P_{k+1} - \Delta P_k | \\
            & \quad + \frac{\log_2(e)\delta}{\epsilon^2} | x_{k+1}-x_k | + \frac{\log_2(e)\delta}{\epsilon^2} | P_{k+1}-P_k |
            \nonumber
        \end{split}
    \end{equation}
    Therefore,
    \begin{equation}
        \begin{split}
            & |\mathcal{D}(x'_k, x'_{k+1}, P'_k, P'_{k+1})-\mathcal{D}(x_k, x_{k+1}, P_k, P_{k+1})| \\
            & \leq \left(1 + \frac{\log_2(e)}{\epsilon}\right) | \Delta x_{k+1} - \Delta x_k | \\
            & \quad + \frac{\log_2(e)}{\epsilon} | \Delta P_{k+1} - \Delta P_k | \\
            & \quad + \frac{\log_2(e)\delta}{\epsilon^2} | x_{k+1}-x_k | + \frac{\log_2(e)\delta}{\epsilon^2} | P_{k+1}-P_k | \label{eq:two}
        \end{split}
    \end{equation}
    
    \underline{Case 2}: Next, we consider the case with \begin{subequations}
    \begin{align}
    \mathcal{D}_{\text{info}}(x_k',x_{k+1}',P_k',P_{k+1}')&>0 \text{ and } \label{eq:13a} \\ \mathcal{D}_{\text{info}}(x_k,x_{k+1},P_k,P_{k+1})&=0. \label{eq:13b}
    \end{align}
    \end{subequations}
    Notice that \eqref{eq:13b} implies $P_k+|x_{k+1}-x_k|-P_{k+1} \leq 0$. In this case:
    \begin{subequations}
        \begin{align}
            & |\mathcal{D}(x'_k, x'_{k+1}, P'_k, P'_{k+1})-\mathcal{D}(x_k, x_{k+1}, P_k, P_{k+1})| \nonumber \\
            & = \bigg| |x_{k+1}+\Delta x_{k+1}-x_k-\Delta x_k|- |x_{k+1}-x_k| \nonumber \\
            & \quad + \left. \frac{1}{2}\log_2\left(P_k + \Delta P_k + \left|x_{k+1}+\Delta x_{k+1}-x_k-\Delta x_k\right| \right) \right. \nonumber \\
            & \quad - \frac{1}{2}\log_2(P_{k+1}+\Delta P_{k+1}) \bigg| \\
            & \leq \left| \Delta x_{k+1}- \Delta x_k \right| - \frac{1}{2}\log_2(P_{k+1}+\Delta P_{k+1}) \nonumber \\
            & \quad + \frac{1}{2}\log_2(P_k + \Delta P_k + |x_{k+1}+\Delta x_{k+1} - x_k-\Delta x_k|) \label{eq:three}\\
            &  \leq \left| \Delta x_{k+1}- \Delta x_k \right| \nonumber \\
            & \quad + \frac{\log_2(e)}{\epsilon}(P_k + \Delta P_k + |x_{k+1}+\Delta x_{k+1} - x_k-\Delta x_k| \nonumber \\
            & \quad - P_{k+1}-\Delta P_{k+1} ) \label{eq:four}\\
            &  = \left| \Delta x_{k+1}- \Delta x_k \right| \nonumber \\
            & \quad + \frac{\log_2(e)}{\epsilon}\! \left( P_k\!+\!\left|x_{k+1}-x_k\right|-P_{k+1}+\Delta P_k - \Delta P_{k+1} \right. \nonumber \\
            & \quad \left. +\left| x_{k+1} +\Delta x_{k+1}-x_k-\Delta x_k \right|-\left|x_{k+1}-x_k \right| \right) \\
            & \leq \left|\Delta x_{k+1}-\Delta x_k \right| \nonumber \\
            & \quad +\frac{\log_2(e)}{\epsilon}\bigg| | x_{k+1} +\Delta x_{k+1}-x_k-\Delta x_k | \\ \nonumber
            & \hspace{2cm}-|x_{k+1}-x_k |  \bigg| \nonumber \\
            & \quad + \frac{\log_2(e)}{\epsilon} \left| \Delta P_{k+1}- \Delta P_k \right| \\
            & \leq \left(1+\frac{\log_2(e)}{\epsilon}\right) \left| \Delta x_{k+1}- \Delta x_k \right| \nonumber \\ 
            &\quad +\frac{\log_2(e)}{\epsilon}\left| \Delta P_{k+1}- \Delta P_k \right| \label{eq:five}
        \end{align}
    \end{subequations}
    In step (\ref{eq:three}), we have used the fact that the difference between the two logarithmic terms is positive, because of the hypothesis \eqref{eq:13a}. In step (\ref{eq:four}), we used the fact that $\log_2 a - \log_2 b \leq \frac{2\log_2(e)}{\epsilon}(a-b)$, for $a>b\geq\frac{\epsilon}{2}$.
    
    \underline{Case 3}: Next, we consider the case with 
    \begin{subequations}
    \begin{align}
        \mathcal{D}_{\text{info}}(x_k',x_{k+1}',P_k',P_{k+1}')&=0 \text{ and } \label{eq:15a}\\ \mathcal{D}_{\text{info}}(x_k,x_{k+1},P_k,P_{k+1})&>0. \label{eq:15b}
    \end{align}
    \end{subequations}
    The first hypothesis \eqref{eq:15a} implies:
    \begin{align} 
        P_k+\Delta P_k + |x_{k+1}+\Delta x_{k+1}-x_k-\Delta x_k|& \nonumber \\
        -P_{k+1}-\Delta P_{k+1}& \leq 0. \label{eq:six} 
    \end{align}
    Using 
    \begin{align*}
        \left| x_{k+1}-x_k\right| &- \left| \Delta x_{k+1}-\Delta x_k\right| \\
        &\leq \left| x_{k+1} + \Delta x_{k+1}-x_k-\Delta x_k\right|,
    \end{align*}
    one can deduce from (\ref{eq:one}) that:
    \begin{equation} \label{eq:seven}
        \begin{split}
            &P_k + |x_{k+1}-x_k|-P_{k+1} \\
            &\leq \Delta P_{k+1} - \Delta P_k+ |\Delta x_{k+1}- \Delta x_k| \\
            &\leq |\Delta P_{k+1} - \Delta P_k|+ |\Delta x_{k+1}- \Delta x_k|.
        \end{split}
    \end{equation}
    This results in:
    \begin{subequations} 
        \begin{align}
            & |\mathcal{D}(x'_k, x'_{k+1}, P'_k, P'_{k+1})-\mathcal{D}(x_k, x_{k+1}, P_k, P_{k+1})| \nonumber \\
            &  = \bigg| \left| x_{k+1}-\Delta x_{k+1} - x_k -\Delta x_k \right| - \left|x_{k+1}-x_k\right|  \nonumber \\
            & \quad -  \frac{1}{2}\log_2(P_k+|x_{k+1}-x_k|) + \frac{1}{2}\log_2(P_{k+1}) \bigg| \\
            & \leq \left| \Delta x_{k+1}-\Delta x_k \right| \nonumber \\
            &\quad + \frac{1}{2}\log_2(P_k+|x_{k+1}-x_k|)  -\frac{1}{2}\log_2(P_{k+1}) \label{eq:eight} \\
            &  \leq \left| \Delta x_{k+1}-\Delta x_k \right| + \frac{\log_2(e)}{2\epsilon}(P_k+|x_{k+1}-x_k|)-P_{k+1}) \label{eq:nine} \\
            & \leq \left( 1 + \frac{\log_2(e)}{2\epsilon}\right)\left| \Delta x_{k+1}-\Delta x_k \right| \nonumber \\
            & \quad + \frac{\log_2(e)}{2\epsilon}\left|\Delta P_{k+1}-\Delta P_{k}\right| \label{eq:ten} 
        \end{align}
    \end{subequations}
    In (\ref{eq:eight}) we used the fact that the difference between the two logarithmic terms is positive. In (\ref{eq:nine}) we used the fact that $\log_2 a - \log_2 b \leq \frac{\log_2(e)}{\epsilon}(a-b)$, for $a>b\geq\epsilon$. Finally, the inequality (\ref{eq:seven}) was used in step (\ref{eq:ten}).
    
    \underline{Case 4}:
    Finally, we consider the case with 
    \begin{align*}
    \mathcal{D}_{\text{info}}(x_k',x_{k+1}',P_k',P_{k+1}')&=0 \text{ and } \\ \mathcal{D}_{\text{info}}(x_k,x_{k+1},P_k,P_{k+1})&=0.
    \end{align*}
 In this case:
    \begin{equation} \label{eq:eleven}
        \begin{split}
            & |\mathcal{D}(x'_k, x'_{k+1}, P'_k, P'_{k+1})-\mathcal{D}(x_k, x_{k+1}, P_k, P_{k+1})| \\
            & \quad =\big|\left| x_{k+1}-\Delta x_{k+1} - x_k -\Delta x_k \right|- |x_{k+1}-x_k| \big| \\
            & \quad \leq |\Delta x_{k+1}- \Delta x_k |
        \end{split}
    \end{equation}
    To summarize, (\ref{eq:two}), (\ref{eq:five}), (\ref{eq:ten}), and (\ref{eq:eleven}) are sufficient to be able to choose $L_\epsilon = \max \{1 + \frac{\log_2(e)}{\epsilon}, \frac{\log_2(e)}{\epsilon^2}\}$ to obtain the desired result.
\end{proof}

\noindent \textbf{\emph{Proof of Theorem \ref{theo:continuity}}} \\
Suppose $\gamma(t)=(x(t), P(t))$ and $\gamma'(t)=(x'(t), P'(t))$. In what follows, we consider the choice:
\begin{equation}
    \delta=\min\left\{\frac{\epsilon_0}{2L_\epsilon\left(1+|\gamma|_{\text{TV}}\right)}, \frac{\epsilon}{2}\right\}
\end{equation}
Since $|\gamma'-\gamma|_{\text{TV}}<\delta$, we have $\|\gamma'-\gamma\|_\infty<\delta$. In particular, for each $t\in[0,T]$, we have $|x'(t)-x(t)|<\delta$ and $|P'(t)-P(t)|<\delta$. Moreover:
\begin{equation}
    \begin{split}
        |P'(t)|&=|P(t)+P'(t)-P(t)| \\
        &\geq |P(t)|-|P'(t)-P(t)|> 2\epsilon-\delta>\epsilon. \nonumber
    \end{split}
\end{equation}
Therefore, $\forall t\in[0,T]$, we have both $P(t)>\epsilon$ and $P'(t)>\epsilon$. Let $\mathcal{P}=(0=t_0<t_1<\cdots < t_N=T)$ be a partition and define:
\begin{equation}
    c(\gamma; \mathcal{P}):=\sum_{k=0}^{N-1} \mathcal{D}(x(t_k), x(t_{k+1}), P(t_k), P(t_{k+1})).\nonumber
\end{equation}
For any partition $\mathcal{P}$, the following chain of inequalities holds:
\begin{subequations}
\label{eq:chain}
\begin{align}
    &|c(\gamma'; \mathcal{P})-c(\gamma; \mathcal{P})| \nonumber \\
    &=\left|\sum_{k=0}^{N-1}\mathcal{D}(x'(t_k), x'(t_{k+1}), P'(t_k), P'(t_{k+1}))\right. \nonumber \\
    & \qquad \qquad -\mathcal{D}(x(t_k), x(t_{k+1}), P(t_k), P(t_{k+1})) \Bigg| \\
    &\leq  \sum_{k=0}^{N-1} \left|\mathcal{D}(x'(t_k), x'(t_{k+1}), P'(t_k), P'(t_{k+1}))\right. \nonumber \\
    &\left. \qquad \qquad -\mathcal{D}(x(t_k), x(t_{k+1}), P(t_k), P(t_{k+1})) \right| \\
    &\leq L_{\epsilon}  \sum_{k=0}^{N-1} \Bigl[
    |(x'(t_{k+1})-x(t_{k+1}))-(x'(t_k)-x(t_k))| \Bigr. \nonumber \\
    & \qquad  \left.+|(P'(t_{k+1})-P(t_{k+1}))-(P'(t_k)-P(t_k))| \right. \nonumber \\
    & \qquad  +\Bigl. \delta | P(t_{k+1})-P(t_k)| + \delta | x(t_{k+1})-x(t_k)|\Bigr]\label{eq:15d} \\
    &= L_{\epsilon}  (V(\gamma'-\gamma,\mathcal{P}) + \delta V(\gamma,\mathcal{P}))  \\
    &\leq L_{\epsilon} \left(|\gamma'-\gamma|_{\text{TV}}+ \delta|\gamma|_{\text{TV}}\right) \\
    &< L_{\epsilon}  \left(\delta+\delta|\gamma|_{\text{TV}}\right) \\
    &\leq L_{\epsilon}  \left(1+|\gamma|_{\text{TV}}\right)\left( \frac{\epsilon_0}{2L_\epsilon(1+|\gamma|_{\text{TV}})} \right) \\
    &=\frac{\epsilon_0}{2} 
\end{align}
\end{subequations}
The inequality (\ref{eq:15d}) follows from Lemma \ref{lem:lipschitz}. Let $\{\mathcal{P}_i\}_{i\in\mathbb{N}}$ and $\{\mathcal{P}'_i\}_{i\in\mathbb{N}}$ be sequences of partitions such that:
\begin{equation}
\label{eq:p_sequence}
    \lim_{i \rightarrow \infty} c(\gamma; \mathcal{P}_i) =c(\gamma), \;\; \lim_{i \rightarrow \infty} c(\gamma'; \mathcal{P}'_i) =c(\gamma'),
\end{equation}
and let $\{\mathcal{P}''_i\}_{i\in\mathbb{N}}$ be the sequence of partitions such that for each $i\in\mathbb{N}$, $\mathcal{P}''_i$ is a common refinement of $\mathcal{P}_i$ and $\mathcal{P}'_i$. Since both
\begin{align*}
& c(\gamma;\mathcal{P}_i) \!\leq\! c(\gamma;\mathcal{P}''_i) \leq c(\gamma)\; \text{and} \; c(\gamma';\mathcal{P}'_i) \leq c(\gamma;\mathcal{P}''_i)  \leq c(\gamma)\nonumber
\end{align*}
hold for each $i\in\mathbb{N}$, \eqref{eq:p_sequence} implies
\begin{equation}
\label{eq:p_sequence2}
\lim_{i \rightarrow \infty} c(\gamma; \mathcal{P}''_i) =c(\gamma), \;\;
\lim_{i \rightarrow \infty} c(\gamma'; \mathcal{P}''_i) =c(\gamma').
\end{equation}
Now, since the chain of inequalities \eqref{eq:chain} holds for any partitions, 
\[
|c(\gamma; \mathcal{P}''_i) - c(\gamma'; \mathcal{P}''_i) | < \frac{\epsilon_0}{2}
\]
holds for all $i\in\mathbb{N}$. This results in:
\begin{align}
|c(\gamma)-c(\gamma')|&=\lim_{i\rightarrow \infty} |c(\gamma; \mathcal{P}''_i) - c(\gamma'; \mathcal{P}''_i) |  \label{eq:17} \nonumber \\
&\leq \frac{\epsilon_0}{2} < \epsilon_0.
\end{align}
where  \eqref{eq:17} follows from \eqref{eq:p_sequence2}.\\

\section{One-Dimensional Problem Optimal Path}\label{sec:AppenIII}
Consider taking the single perception optimal path $\gamma_1$ in Fig. \ref{fig:1D_tree}~(a): 
\begin{equation} 
    (x_{0},P_{0}) \rightarrow (x_{T},P_{T}) \nonumber
\end{equation}
and dividing it into the combination of two sub-paths $\gamma_2$:
\begin{equation} 
\begin{split}
    &(x_{0},P_{0}) \rightarrow (x_{a},P_{a}) \rightarrow (x_{T},P_{T}) \nonumber \\
    &\text{such that} \; \; \hat{P}_0^{'}=P_0+ \beta \|x_{T}-x_0\| W> P_a, \nonumber \\ 
    & \qquad \qquad \hat{P}_a^{'}=P_a+ (1-\beta)\|x_{T}-x_0\| W > P_T \nonumber
\end{split}
\end{equation} 
where $\beta \in (0,1)$ is a constant which denotes where in $\gamma_2$ the additional sensing action takes place. The combination of the divided sub-paths have in the same initial ($z_{0}$) and ending ($z_{T}$) states as the original path, but also achieve an intermediate state ($z_{a}$).

Path $\gamma_1$ has a total RI cost:
\begin{subequations}
\begin{align}
    \mathcal{D}(\gamma_1)&=\|x_{T}-x_0\|+\frac{\alpha}{2}\left[\log_2\hat{P}_0 -\log_2 P_{T} \right] \nonumber,
\end{align}
\end{subequations}
where, $\hat{P}_0= P_0 +\|x_{T}-x_0\| W$. Likewise, the path $\gamma_2$ has a length which is the summation of two information gains and while transitioning the same distance as $\gamma_1$.
\begin{subequations}
\begin{align}
    &\mathcal{D}(\gamma_2)=\|x_{T}-x_0\|+\frac{\alpha}{2}\left[\log_2\hat{P}_{0}^{'} -\log_2 P_{a}  \right] \nonumber \\
    & \qquad \qquad +\frac{\alpha}{2}\left[\log_2\hat{P}_{a}^{'} -\log_2 P_{T} \right]\nonumber
\end{align}
\end{subequations}
By comparing the costs between $\gamma_1$ and $\gamma_2$, it is possible to achieve:
\begin{subequations}\label{eq:tot_dist}
\begin{align}
        &\mathcal{D}(\gamma_1) - \mathcal{D}(\gamma_2)= \frac{\alpha}{2}\left[\log_2\hat{P}_0 -\log_2 P_{T} \right] \nonumber \\ 
        & \qquad-\frac{\alpha}{2}\left[\log_2\hat{P}_{0}^{'} -\log_2 P_{a}  \right] - \frac{\alpha}{2}\left[\log_2\hat{P}_{a'} -\log_2 P_{T} \right] \nonumber \\
        &= \frac{\alpha}{2} \left[ \log_2 \hat{P}_0 - \log_2\hat{P}_{0}^{'} \right] - \frac{\alpha}{2} \left[\log_2\hat{P}_{a'} -\log_2 P_{a} \right] \nonumber \\
        & = f(\hat{P}_{0}^{'})-f(P_a) < 0, \nonumber 
    \end{align}
\end{subequations}
where $f(P)\!=\!\frac{\alpha}{2} \left[ \log_2 (P+(1-\beta)\|x_{T}-x_0\| W) \!-\! \log_2 P \right] $. Last inequality follows the facts that $\hat{P}_{0}^{'} > P_a $ and $f(P)$ is a decreasing function ($\dv{f}{P}<0$).

\end{appendices}

\addtolength{\textheight}{-0cm}   



\bibliographystyle{IEEEtran}
\bibliography{ref_shorten.bib}

\begin{thebibliography}{10}
\providecommand{\url}[1]{#1}
\csname url@samestyle\endcsname
\providecommand{\newblock}{\relax}
\providecommand{\bibinfo}[2]{#2}
\providecommand{\BIBentrySTDinterwordspacing}{\spaceskip=0pt\relax}
\providecommand{\BIBentryALTinterwordstretchfactor}{4}
\providecommand{\BIBentryALTinterwordspacing}{\spaceskip=\fontdimen2\font plus
\BIBentryALTinterwordstretchfactor\fontdimen3\font minus
  \fontdimen4\font\relax}
\providecommand{\BIBforeignlanguage}[2]{{%
\expandafter\ifx\csname l@#1\endcsname\relax
\typeout{** WARNING: IEEEtran.bst: No hyphenation pattern has been}%
\typeout{** loaded for the language `#1'. Using the pattern for}%
\typeout{** the default language instead.}%
\else
\language=\csname l@#1\endcsname
\fi
#2}}
\providecommand{\BIBdecl}{\relax}
\BIBdecl

\bibitem{Pendleton2017}
S.~Pendleton, H.~Andersen, X.~Du, X.~Shen, M.~Meghjani, Y.~Eng, D.~Rus, and
  M.~Ang, ``Perception, planning, control, and coordination for autonomous
  vehicles,'' \emph{Machines}, vol.~5, no.~1, p.~6, 2017.

\bibitem{Pfeiffer2017}
M.~Pfeiffer, M.~Schaeuble, J.~Nieto, R.~Siegwart, and C.~Cadena, ``From
  perception to decision: A data-driven approach to end-to-end motion planning
  for autonomous ground robots,'' in \emph{Proc. IEEE Int. Conf. Robot.
  Autom.}, 2017, pp. 1527--1533.

\bibitem{Carlone2019}
L.~Carlone and S.~Karaman, ``Attention and anticipation in fast visual-inertial
  navigation,'' \emph{IEEE Trans. Robot.}, vol.~35, no.~1, pp. 1--20, 2019.

\bibitem{Alterovitz2016}
R.~Alterovitz, S.~Koenig, and M.~Likhachev, ``Robot planning in the real world:
  research challenges and opportunities,'' \emph{AI Magazine}, vol.~37, no.~2,
  pp. 76--84, 2016.

\bibitem{kuwata2008motion}
Y.~Kuwata, J.~Teo, S.~Karaman, G.~Fiore, E.~Frazzoli, and J.~How, ``Motion
  planning in complex environments using closed-loop prediction,'' \emph{AIAA
  Guid. Navi. Control Conf. Exhibit}, 2008.

\bibitem{van2011lqg}
J.~Van Den~Berg, P.~Abbeel, and K.~Goldberg, ``{LQG-MP}: Optimized path
  planning for robots with motion uncertainty and imperfect state
  information,'' \emph{Int. J. Robot. Res.}, vol.~30, no.~7, pp. 895--913,
  2011.

\bibitem{agha2014firm}
A.-A. Agha-Mohammadi, S.~Chakravorty, and N.~M. Amato, ``{FIRM}: Sampling-based
  feedback motion-planning under motion uncertainty and imperfect
  measurements,'' \emph{Int. J. Robot. Res.}, vol.~33, no.~2, pp. 268--304,
  2014.

\bibitem{Sims2003}
C.~A. Sims, ``Implications of rational inattention,'' \emph{J. Monetary
  Economics}, vol.~50, no.~3, pp. 665--690, 2003.

\bibitem{Shafieepoorfard2016}
E.~Shafieepoorfard, M.~Raginsky, and S.~P. Meyn, ``Rationally inattentive
  control of markov processes,'' \emph{SIAM J. Control Optimization}, vol.~54,
  no.~2, pp. 987--1016, 2016.

\bibitem{Shafieepoorfard2013}
E.~Shafieepoorfard and M.~Raginsky, ``Rational inattention in scalar lqg
  control,'' in \emph{Proc. Conf. Decision Control}, 2013, pp. 5733--5739.

\bibitem{Lavalle2006}
S.~M. LaValle, \emph{Planning algorithms}.\hskip 1em plus 0.5em minus
  0.4em\relax Cambridge {U}niversity {P}ress, 2006.

\bibitem{Karaman2010}
S.~Karaman and E.~Frazzoli, ``Incremental sampling-based algorithms for optimal
  motion planning,'' in \emph{Proc. Robot.: Sci. Syst.}, 2010.

\bibitem{marquez2008design}
J.~J. Marquez and M.~L. Cummings, ``Design and evaluation of path planning
  decision support for planetary surface exploration,'' \emph{J. Aerosp.
  Comput., Info., Comm.}, vol.~5, no.~3, pp. 57--71, 2008.

\bibitem{Hwang1992}
Y.~K. Hwang and N.~Ahuja, ``A potential field approach to path planning,''
  \emph{IEEE Trans. Robot. Autom.}, vol.~8, no.~1, pp. 23--32, 1992.

\bibitem{Kambhampati1986}
S.~Kambhampati and L.~Davis, ``Multiresolution path planning for mobile
  robots,'' \emph{IEEE J. Robot. Autom.}, vol.~2, no.~3, pp. 135--145, 1986.

\bibitem{Hauer2015}
F.~Hauer, A.~Kundu, J.~M. Rehg, and P.~Tsiotras, ``Multi-scale perception and
  path planning on probabilistic obstacle maps,'' in \emph{Proc. IEEE Int.
  Conf. Robot. Autom.}, 2015, pp. 4210--4215.

\bibitem{Lambert2003}
A.~Lambert and D.~Gruyer, ``Safe path planning in an uncertain-configuration
  space,'' in \emph{Proc. IEEE Int. Conf. Robot. Autom.}, vol.~3, 2003, pp.
  4185--4190.

\bibitem{Pepy2006}
R.~Pepy and A.~Lambert, ``Safe path planning in an uncertain-configuration
  space using {RRT},'' in \emph{Proc. IEEE/RSJ Int. Conf. Intell. Robots
  Syst.}, 2006, pp. 5376--5381.

\bibitem{Kloeden1992}
P.~Kloeden and E.~Platen, \emph{Numerical Solution of Stochastic Differential
  Equations}.\hskip 1em plus 0.5em minus 0.4em\relax Berlin: Springer, 1992.

\bibitem{lambert2003safe}
A.~Lambert and D.~Gruyer, ``Safe path planning in an uncertain-configuration
  space,'' in \emph{Proc. IEEE Int. Conf. Robot. Autom.}, vol.~3, 2003, pp.
  4185--4190.

\bibitem{pepy2006safe}
R.~Pepy and A.~Lambert, ``Safe path planning in an uncertain-configuration
  space using {RRT},'' in \emph{Proc. IEEE/RSJ Int. Conf. Intell. Robots
  Syst.}, 2006, pp. 5376--5381.

\bibitem{tanaka2016semidefinite}
T.~Tanaka, K.-K. Kim, P.~A. Parrilo, and S.~K. Mitter, ``Semidefinite
  programming approach to {G}aussian sequential rate-distortion trade-offs,''
  \emph{IEEE Trans. Automat. Control}, vol.~62, no.~4, pp. 1896--1910, 2016.

\bibitem{vandenberghe1998determinant}
L.~Vandenberghe, S.~Boyd, and S.-P. Wu, ``Determinant maximization with linear
  matrix inequality constraints,'' \emph{SIAM J. Matrix Analysis and
  Applications}, vol.~19, no.~2, pp. 499--533, 1998.

\bibitem{carothers2000real}
N.~L. Carothers, \emph{Real analysis}.\hskip 1em plus 0.5em minus 0.4em\relax
  Cambridge {U}niversity {P}ress, 2000.

\bibitem{Lavalle2001}
S.~M. LaValle and J.~J. Kuffner, ``Randomized kinodynamic planning,''
  \emph{Int. J. Robot. Res.}, vol.~20, no.~5, pp. 378--400, May 2001.

\bibitem{karaman2011anytime}
S.~Karaman, M.~R. Walter, A.~Perez, E.~Frazzoli, and S.~Teller, ``Anytime
  motion planning using the {RRT},'' in \emph{Proc. IEEE Int. Conf. Robot.
  Autom.}, 2011, pp. 1478--1483.

\bibitem{Karaman2011}
S.~Karaman and E.~Frazzoli, ``Sampling-based algorithms for optimal motion
  planning,'' \emph{Int. J. Robot. Res.}, vol.~30, no.~7, pp. 846--894, 2011.

\end{thebibliography}

\end{document}